\documentclass[a4paper]{amsart}
\usepackage{amssymb,amsmath,amsthm}
\usepackage{mathrsfs}
\usepackage{color}
\usepackage[dvipsnames]{xcolor}
\usepackage{url}
\usepackage{multirow}
\usepackage[top=4cm, bottom=3cm, right=3cm, left=3cm]{geometry}
\usepackage{setspace}
\usepackage{enumerate}
\usepackage{algpseudocode}

\usepackage{hyperref}
\hypersetup{
  colorlinks   = true,    
  urlcolor     = blue,    
  linkcolor    = blue,    
  citecolor    = red,     
}

\usepackage[utf8]{inputenc} 
\usepackage[T1]{fontenc}    
\usepackage{url}            

\usepackage{amsfonts}       

\usepackage{subcaption}
\usepackage[ruled,vlined]{algorithm2e}
\usepackage{bbm}

\def\R{\mathbb{R}}
\newcommand{\be}{\begin{equation}}
\newcommand{\en}{\end{equation}}
\newcommand{\ben}{\begin{equation*}}
\newcommand{\enn}{\end{equation*}}
\newcommand{\bea}{\begin{aligned}}
\newcommand{\ena}{\end{aligned}}

\def\P{\mathcal{P}}
\def\eps{\epsilon}
\def\gt{\rightarrow}
\def\wgt{\rightharpoonup}

\def\M{\mathcal{M}}

\def\F{\mathcal{F}}

\theoremstyle{plain}
\newtheorem{thm}{Theorem}[section]

\newtheorem{prop}[thm]{Proposition}

\theoremstyle{remark}

\usepackage{todonotes}

\newcommand{\commentout}[1]{}

\begin{document}
\title{Accelerating Langevin Sampling with Birth-death}

\author[Y. Lu]{Yulong Lu}
\address[Y. Lu]{Department of Mathematics, Duke University, Durham NC 27708, USA}
\email{yulonglu@math.duke.edu}

\author[J. Lu]{Jianfeng Lu}
\address[J. Lu]{Department of Mathematics,
  Department of Physics and Department of Chemistry,  Duke University, Durham NC 27708, USA}
\email{jianfeng@math.duke.edu}

\author[J. Nolen]{James Nolen}
\address[J. Nolen]{Department of Mathematics, Duke University, Durham NC 27708, USA}
\email{nolen@math.duke.edu}

\begin{abstract}
  A fundamental problem in Bayesian inference and statistical
  machine learning is to efficiently sample from multimodal
  distributions.  Due to metastability, multimodal distributions are
  difficult to sample using standard Markov chain Monte Carlo methods.
  We propose a new sampling algorithm based on a birth-death mechanism
  to accelerate the mixing of Langevin diffusion. Our algorithm is
  motivated by its mean field partial differential equation (PDE), which is a Fokker-Planck
  equation supplemented by a nonlocal birth-death term.  This PDE can
  be viewed as a gradient flow of the Kullback-Leibler divergence with
  respect to the Wasserstein-Fisher-Rao metric.  We prove that under
  some assumptions the asymptotic convergence rate of the nonlocal PDE
  is independent of the potential barrier, in contrast to the exponential
  dependence in the case of the Langevin diffusion. We illustrate the
  efficiency of the birth-death accelerated Langevin method through
  several analytical examples and numerical experiments.
\end{abstract}

\maketitle

\section{Introduction}
%

Numerical sampling of high dimensional probability distributions with unknown normalization has
important applications in machine learning, Bayesian statistics,
computational physics, and other related fields. The most popular
approaches are based on Markov chain Monte Carlo (MCMC), including
Langevin MCMC \cite{RT96}, underdamped Langevin MCMC
\cite{GrestKremer:86}, Hamiltonian Monte Carlo
\cite{RosskyDollFriedman:78,Anderson:80,Neal:11}, bouncy particle and
zigzag samplers \cite{BCVollmerDoucet:18,BierkensFearnheadRoberts:19},
etc. Many of these approaches, often combined with stochastic
gradient \cite{WellingTeh:11, MaChenFox:15}, have been widely used in machine learning.

When the target probability distribution is (strongly) log-concave,
that is, when its density with respect to the Lebesgue measure is
$\pi(x) \propto e^{-V(x)}$ with $V$ being (strongly) convex, it is
known that most sampling schemes mentioned above can produce
independent random samples efficiently
\cite{dalalyan2017theoretical,durmus2017nonasymptotic,MangoubiSmith:17,ChengChatterji:18,DwivediChenWainwrightYu:18}. The
sampling problem becomes much more challenging when the probability
distribution exhibits multi-modality, as
it takes much longer time for the sampling Markov chain to get
through low-probability regions in the phase space to explore and
balance between multiple modes (also known as metastability).  Many enhanced sampling schemes have
been proposed over the years to overcome such difficulty, including
various tempering schemes \cite{SwendsenWang:86, Geyer:91,
  MarinariParisi:92, Neal:01}, biasing techniques
\cite{WangLandau:01,LaioParrinello:02}, non-equilibrium sampling
\cite{Nilmeier:11}, just to name a few.

In this work, we propose a simple, novel sampling dynamics to overcome
the metastability based on birth-death process. On the continuous
level of evolution of the probability density, the proposed dynamics
is given by
\begin{equation}
  \partial_t \rho_t = \underbrace{\nabla \cdot \bigl(\nabla \rho_t + \rho_t \nabla V\bigr)}_{\text{overdamped Langevin}} + \underbrace{\rho_t \bigl( \log \pi - \log \rho_t \bigr)  - \rho_t \mathbb{E}_{\rho_t} \bigl(\log\pi - \log \rho_t\bigr)}_{\text{birth-death}},
\end{equation}
where one adds a non-local (due to the expectation) update rule to the
conventional overdamped Langevin dynamics. The advantage of the
birth-death process is that it allows global move of the mass of a probability density
directly from one mode to another in the phase space according to
their relative weights, without the difficulty of going through low
probability regions, suffered by any local dynamics such as the
overdamped Langevin MCMC. It is possible to
combine the birth-death process with other sampling dynamics, such as
underdamped Langevin or various accelerated dynamics, while for
simplicity of presentation we will focus only on the birth-death accelerated overdamped
Langevin dynamics in the current work.

\subsection{Contribution.}
Our main theoretical result is that under mild assumptions the
asymptotic convergence rate of the proposed sampling dynamics is
independent of the barrier of the potential corresponding to
the target measure -- this is a substantial improvement of the convergence rate of overdamped Langevin diffusion, which is exponentially
small due to metastability. Moreover, we also establish a gradient flow structure
of the birth-death accelerated Langevin dynamics: it is a gradient
flow of the Kullback-Leibler (KL)-divergence with respect to the
Wasserstein-Fisher-Rao metric.

To demonstrate this improved convergence, we study two
analytical examples showing significant speedup of mixing compared
to Langevin diffuion. We also propose a practical interacting
particle sampling scheme as a numerical implementation of the
birth-death accelerated Langevin dynamics.  The efficiency of the
proposed algorithm is illustrated through several numerical examples.

\subsection{Related works.}
The proposed sampling scheme involves interacting particles that
undergo Langevin diffusion and birth-death process.  Other sampling
schemes via interacting particles have been proposed recently,
including the Stein variational gradient descent (SVGD) flow
\cite{LiuWang16,Liu17} (see also its continuous limit studied in
\cite{LLN19}). Unlike the samplers based on Stein discrepancy, which
replaces the random noise in Langevin dynamics by repulsion of
particles, the sampling scheme proposed in this work employs
birth-death process to enhance the mixing of existing sampling
schemes. In fact, it can also be potentially combined with SVGD to
improve its convergence.

The birth-death process has been used in sequential Monte Carlo (SMC)
samplers \cite{DelMoralDoucetJasra:06}. In SMC, the
birth-death and branching process is used to reduce the variance of
particle weights. While in the current proposed scheme, it is used to
globally move the sampling particles according to the target
measure. The birth-death process is also used recently to accelerate
training of neural networks in the mean-field regime
\cite{Rotskoff2019neuron}, also for a quite different purpose than
accelerating convergence of Monte Carlo samplers.

\subsection*{Acknowledgement}

The work of Jianfeng Lu and the work of James Nolen were partially funded through grants DMS-1454939 and DMS-1351653 from the National Science Foundation, respectively.

\section{Fokker-Planck equation and birth-death process}
\subsection{Langevin dynamics and its Fokker-Planck equation.}
Recall the (overdamped) Langevin diffusion is the solution to the following stochastic differential equation
\be\label{eq:oLan}
d X_t = -\nabla V(X_t)dt  + \sqrt{2}dW_t,
\en
where $X_t \in \R^d$ and $W_t$ is a $d$-dimensional Brownian motion.
Many popular sampling schemes are constructed from discretizations of \eqref{eq:oLan},
such as the unadjusted Langevin algorithm (ULA) and its Metropolized version -- Metropolis-adjusted Langevin algorithm (MALA) \cite{RT96}.
The probability density function $\rho_t(x)$ of  \eqref{eq:oLan}
solves the linear Fokker-Planck equation (FPE) on $\mathbb{R}^d$
\be\label{eq:fp1}
\partial_t \rho_t = \nabla \cdot (\nabla \rho_t + \rho_t \nabla V).
\en
The stationary distribution of  \eqref{eq:oLan} and \eqref{eq:fp1} has  density $\pi(x) = e^{-V(x)}/Z$.
In the seminal work of Jordan-Kinderlehrer-Otto \cite{JKO}, the FPE \eqref{eq:fp1} was identified as the gradient flow of the KL-divergence  (\textit{i.e.}, relative entropy)
$\mathrm{KL}(\rho | \pi) = \int \rho \log(\rho/\pi) \,dx$,
with respect to the $2$-Wasserstein distance.
Moreover, if the target measure $\pi$ satisfies
 the logarithmic Sobolev inequality (LSI): for any probability distribution $\rho$,
\be
\label{eq:log-Sobolev}
\mathrm{KL}(\rho | \pi) \leq \frac{1}{\lambda}\mathcal{I}(\rho | \pi) \text{ with the relative Fisher information } \mathcal{I}(\rho | \pi) = \int \rho |\nabla \log (\rho/\pi)|^2 \,dx,
\en
then we have the exponential convergence $\mathrm{KL}(\rho_t | \pi)  \leq e^{-\lambda t} \mathrm{KL}(\rho_0 | \pi)$.
The convergence rate $\lambda$ above may be exponentially small though when the target distribution is multimodal with high potential barrier; see \cite{BEGK04, BGK05}.

\subsection{Pure birth-death process.}
The main idea of this work is to use birth-death process to accelerate
sampling. Before combining it with the Langevin dynamics, let us
consider the pure birth-death equation (BDE) given by
\begin{equation}\label{eq:fp2}
\partial_t \rho_t = -\alpha_t \rho_t, \quad  \textrm{ with } \quad \alpha_t(x) := \log \rho_t(x) - \log \pi(x) - \int_{\R^d} (\log \rho_t - \log \pi)\rho_t dy.
\end{equation}
As there is no spatial derivative involved in \eqref{eq:fp2}, it can be viewed as a (infinite) system of ordinary differential equations, indexed by $x \in \mathbb{R}^d$, coupled through the integral term in $\alpha_t$.  Observe that $\pi$ is an invariant measure of  \eqref{eq:fp2}.
Moreover, equation \eqref{eq:fp2} depends on $\pi$ only up to a multiplicative constant,
making it feasible for sampling $\pi$ with an unknown normalization constant.
The definition of the birth/death rate $\alpha_t(x)$  in \eqref{eq:fp2} is very intuitive. In fact,
ignoring the last integral term in the definition of $\alpha_t$, one sees that the solution $\rho_t$ to \eqref{eq:fp2}
adjusts the mass according to the difference of the logarithm of current density and that of the target:
the density $\rho_t(x)$ at a location $x$ increases (or decreases) if
$\rho_t(x) < \pi(x)$ (or $\rho_t(x) > \pi(x)$). The integral term in \eqref{eq:fp2} is added to guarantee that the total integral of $\rho_t$ is conserved
during the evolution, and thus $\rho_t$ remains a probability distribution (positivity is also to verify).

The birth-death dynamics \eqref{eq:fp2} differs substantially from FPE
\eqref{eq:fp1} in many aspects. The former is essentially a nonlinear
system of ODEs (but with a non-local coefficient) whereas the later is
a linear parabolic PDE.  Due to the absence of diffusion, the support
of the solution $\rho_t$ of \eqref{eq:fp2} never increases during the
evolution.
This seems suggesting that the birth-death equation is
unsuitable for sampling.
However, we shall show in Theorem \ref{thm:conv1} that
if the initial density is positive everywhere, then $\rho_t$ converges to $ \pi$  as $t\rightarrow \infty$.

\subsection{Fokker-Planck equation with birth-death dynamics.}

The real power of the birth-death process above comes in when it is combined with the Fokker-Planck equation \eqref{eq:fp1}
, which yields the following equation on the level of probability density, already appeared in the introduction:
\be
\label{eq:fp3}
\partial_t \rho_t = \nabla \cdot (\nabla \rho_t + \rho_t \nabla V) -\alpha_t \rho_t,
\en
where $\alpha_t = \log \rho_t - \log \pi - \int_{\R^d} (\log \rho_t - \log \pi)\rho_t dx$.
Before we discuss the discretization of \eqref{eq:fp3} in Section~\ref{sec:particle},
which will lead to an efficient particle sampler in practice, in what follows, we study the
Fokker-Planck equation of birth-death accelerated Langevin dynamics (BDL-FPE) \eqref{eq:fp3},
in particular its favorable convergence properties compared to the original Langevin dynamics \eqref{eq:fp1} and the pure birth-death process \eqref{eq:fp2}.

\section{Analysis of the Fokker-Planck equation with birth-death}
\subsection{Gradient flow structure} In parallel to well-known fact
that FPE \eqref{eq:fp1} is the 2-Wasserstein gradient flow of the
KL-divergence, BDL-FPE \eqref{eq:fp3} can be viewed as a gradient flow
of the KL-divergence with respect to a different metric.  Our result
is motivated by recent works on the Wasserstein-Fisher-Rao (WFR)
distance \cite{KMV16, CPSV18, LMS16, LMS18} in the study of unbalanced
optimal transport. Specifically, we define the Wasserstein-Fisher-Rao
distance (also known as the spherical Hellinger-Kantorovich distance
\cite{KMV16, brenier2018optimal}) by \be\label{eq:dwfr}
d^2_{\textrm{WFR}} (\rho_0, \rho_1) = \inf_{\rho_t, u_t \in
  \mathcal{A}_{\textrm{WFR}}(\rho_1, \rho_1)} \int_0^1 \Big(\int
|\nabla u_t|^2 + |u_t|^2 d\rho_t - \Big(\int u_t \rho_t dx
\Big)^2\Big)dt, \en where the admissible set
$\mathcal{A}_{\textrm{WFR}} (\rho_0, \rho_1)$ consists of all pairs
$(\rho_t, u_t) \in \P(\R^d)\times L^2(\R^d, d\rho_t)$ such that
$\{\rho_t\}_{t\in[0,1]}$ is a narrowly continuous curve in
$\mathcal{P}(\R^d)$ connecting $\rho_0$ and $\rho_1$ and that
\begin{equation}\label{eq:cont2}
\partial_t \rho_t = \nabla\cdot (\rho_t \nabla u_t)-\rho_t (u_t - \int_{\R^d} u_t \rho_t dx ) \textrm{ in the weak sense}.
\end{equation}
Here $ \P(\R^d)$ denotes the space of probability measures on $\R^d$ and $L^2(\R^d, d\rho_t)$ is the space of functions $u$ satisfying $\int u^2(x) \rho_t(x)dx < \infty$.
We emphasize that for our sampling purpose we have modified the
original definition of WFR distance in \cite{BBM16,KMV16, CPSV18} by
adding the integral penalty term to keep mass conserved.
Without this term, $\rho_t$ may experience gain and loss of mass
during transportation procedure. Our first result is the following
theorem which characterizes the gradient flow structure of  BDL-FPE
\eqref{eq:fp3}, whose proof is provided in Appendix \ref{sec:gradientflow}.
\begin{thm}\label{thm:gf}
  The Fokker-Planck equation for birth-death accelerated Langevin (BDL-FPE)
  dynamics \eqref{eq:fp3} is the gradient flow of the KL-divergence
  $\mathrm{KL}(\cdot | \pi)$ with respect to the
  Wasserstein-Fisher-Rao distance \eqref{eq:dwfr}.
\end{thm}
As a consequence of Theorem \ref{thm:gf}, the dynamics \eqref{eq:fp3}
dissipates the KL-divergence in a steepest descent manner with respect
to the WFR metric \eqref{eq:dwfr}, similar to the variational structure for the
Fokker-Planck equation \eqref{eq:fp1} (w.r.t. the 2-Wasserstein metric).

\subsection{Convergence analysis}

We now analyze the convergence of BDL-FPE~\eqref{eq:fp3}. Proofs of
results in this section can be found in Appendix \ref{sec:proofcov}. We first
establish in the following theorem the global convergence of
\eqref{eq:fp3} by assuming the validity of LSI \eqref{eq:log-Sobolev}.
\begin{thm}\label{thm:conv2}
  Assume that $\pi$ satisfies the log-Sobolev inequality
  \eqref{eq:log-Sobolev} with constant $\lambda > 0$. Then the
  solution $\rho_t$ to BDL-FPE \eqref{eq:fp3} with initial condition $\rho_0$
  satisfies \be \mathrm{KL}(\rho_t | \pi) \leq e^{-\lambda t}
  \mathrm{KL}(\rho_0 | \pi).  \en
\end{thm}

Theorem \ref{thm:conv2} shows that the global convergence rate of
BDL-FPE \eqref{eq:fp3} can be no worse than that of FPE \eqref{eq:fp1}.
The convergence rate obtained this way is fully characterized by
the log-Sobolev constant though, which may scale badly when the potential $V$ has high potential barriers.
In contrast, we show in the next theorem that the birth-death term accelerates the diffusion dramatically in the sense that
the asymptotic convergence rate of BDL-FPE \eqref{eq:fp3} is independent of the potential barrier of $V$.
\begin{thm} \label{thm:converge2}
Let $\rho_t$ solve \eqref{eq:fp3} for $t \geq t_0$, with initial condition satisfying $\mathrm{KL}(\rho_{t_0} | \pi)  \leq 1$. Suppose that for some $M \geq 1$,
\begin{equation}
\inf_{x \in \mathbb{R}^d} \frac{\rho_{t_0}(x)}{\pi(x)} \geq e^{-M} \label{fLower1}
\end{equation}
also holds.  Then, for any $\delta \in (0,1/4)$,
\be
\label{eq:kl-bd}
\mathrm{KL}(\rho_t | \pi)  \leq e^{-(2-3\delta)(t - t_*)} \mathrm{KL}(\rho_{t_0} | \pi)
\en
holds for all $t \geq t_* = t_0 + \log\left( \frac{M}{\delta^3} \right)$. In particular, the BDL-FPE \eqref{eq:fp3} has an asymptotic convergence rate which is independent of the potential $V$ corresponding to $\pi$.
\end{thm}
Theorem \ref{thm:converge2} states that as long as the solution $\rho_t$ of
BDL-FPE \eqref{eq:fp3} is not too far from the target ($\mathrm{KL}(\rho_t| \pi) \leq 1$) and satisfies a uniform lower bound (maybe tiny) with respect to $\pi$ starting from $t_0$,
it will converge to $\pi$ with a rate independent of $V$ after a short waiting time. In practice, $t_0$ can be chosen $O(1)$ to satisfy the condition
\eqref{fLower1}; see Section \ref{sec:example} for examples.

For completeness, we also
show the global convergence of BDE \eqref{eq:fp2} (without a rate) in the next theorem. The BDE also has a similar gradient flow structure; we will not go into details here.
\begin{thm}\label{thm:conv1}
 Let $\rho_t$ be the solution to \eqref{eq:fp2} with initial condition $\rho_0$. Assume that $\log \pi(x)$ is finite for any $x\in \R^d$. Assume also that
 $\rho_0$ satisfies that $\mathrm{KL}(\rho_0 | \pi) < \infty$  and $\rho_0(x) > 0$ for all $x\in \R^d$. Then
 $\rho_t(x) > 0$ for all $x\in \R^d$ and $t > 0$. Moreover, for all $x\in \R^d$, $\lim_{t\gt\infty}\rho_t(x) = \pi(x)$ and $\lim_{t\gt \infty}\mathrm{KL}(\rho_t | \pi)  = 0$.
\end{thm}

\section{Illustrative examples}\label{sec:example}

Here we present two very simple examples illustrating how the combined dynamics of BDL-FPE \eqref{eq:fp3} may significantly enhance convergence to equilibrium, compared to either FPE \eqref{eq:fp1} or BDE \eqref{eq:fp2}.

\subsection{Uniform distribution on torus.}
Let $L >\!\!> 1$, and suppose the domain is the $d$-dimensional torus
$\mathbb{T}^d_L = [0,L]^d$ with $\pi(x) \equiv L^{-d}$ being the
density of the uniform measure on $\mathbb{T}^d_L$.  In this case,
FPE dynamics \eqref{eq:fp1} corresponds to the heat equation
$\partial_t w = \Delta w$ on $\mathbb{T}_L^d$.  The spectral gap is
$O(L^{-2})$, and hence the rate of convergence to the equilibrium
measure is $O(L^{-2})$. While this convergence rate is slow for large
$L$, the FPE dynamics \eqref{eq:fp1} may be used to prepare
a good initial condition for the combined BDL-FPE dynamics \eqref{eq:fp3}.
Specifically, a lower bound on the heat kernel shows that at time
$t = 1$ the solution to $\partial_t w = \Delta w$ will satisfy
\[
\inf_{x \in \mathbb{T}_L^d} w(1,x) \geq (c_1)^{d/2} e^{- d L^2 }.
\]
for a universal constant $c_1 > 0$, that is independent of the initial data (assuming it is a probability measure) and the dimension $d$. Then, if we use $\rho_1(x) = w(1,x)$ as initial data for the combined dynamics \eqref{eq:fp3}, for $t \geq t_0 = 1$, the condition \eqref{fLower1} holds with $M = O(d L^2)$.  If $\mathrm{KL}(\rho_{t_0} | \pi)  \leq 1$ also holds, Theorem \eqref{thm:converge2} then implies $\mathrm{KL}(\rho_t | \pi) \leq e^{- (t - t_*)}$ for $t \geq t_* = t_0 + O(\log d) + O(\log L)$.  In particular, the convergence rate does not depend on $L$ and the time lag $(t_* - t_0)$ is $O(\log L)$ rather than $O(L^2)$.

\subsection{Double well.}

Suppose $\pi(x) = Z^{-1} e^{-V(x)}$ with $V(x) = \epsilon^{-1} \cos^2(\pi x)$, $x \in [-1,1]$, for some $\epsilon > 0$.
Here we regard $\pi(x)$ as a density on the one-dimensional torus $\mathbb{T}^1 = [-1,1]$.
This density has two modes at $x = \pm 1/2$.  Moreover, $\max V - \min V = \epsilon^{-1}$.
It is known that for this potential $V$, the FPE dynamics \eqref{eq:fp1} exhibits a metastability phenomenon,
and the mixing time for pure Langevin dynamics is $O(e^{C \epsilon^{-1}})$ for $\epsilon < \!\!< 1$ (see \cite{BEGK04, BGK05, MS14}).

Suppose that the initial density $\rho_0$ is the restriction of $\pi(x)$ to the region $[-1,0]$:
\begin{equation}
\rho_0(x) = 2 \pi(x) \mathbbm{1}_{[-1,0]}. \label{rhoinit}
\end{equation}
Then $\mathrm{KL}(\rho_0 | \pi) = \log 2 < 1$.  If $\rho_t$ evolves according to pure Langevin dynamics \eqref{eq:fp1} for $t \in [0,1]$, a lower bound on the heat kernel (via a large deviation type estimate \cite{FW84}, or by \cite{Nor97}) implies that at $t = 1$,
\[
\inf_{x \in [-1,1]} \rho_1(x) \geq C_1 e^{- C_2 \epsilon^{-2}}
\]
for some postive constants $C_1$ and $C_2$. Then, suppose that for $t \geq t_0 = 1$, $\rho_t$ evolves according to the combined dynamics \eqref{eq:fp3}.  Theorem \eqref{thm:converge2} implies that for $t \geq t_* = 1 + O(|\log \epsilon|)$, we have $\mathrm{KL}(\rho_t | \pi) \leq e^{- (t - t_*)}$.   So, compared to the solution to the Langevin dynamics \eqref{eq:fp1},
the birth-death accelerated dynamics \eqref{eq:fp3} exhibits a dramatic acceleration and converges to $\pi(x)$ at a rate that is independent of $\epsilon$ after a brief delay of $O(\log \epsilon)$.

\section{An interacting particle implementation}\label{sec:particle}
As we mentioned earlier, the FPE \eqref{eq:fp1} has
a nice particle interpretation since it is the probability density
function of the Langevin diffusion \eqref{eq:oLan}.  The dynamics of
BDL-FPE \eqref{eq:fp3} does not have such a simple particle
interpretation, due to the logarithmic nonlinearity in the birth-death
term.  To resolve this difficulty, given a smooth kernel function
$K(x)$ approximating the Dirac delta, we might approximate
\eqref{eq:fp3} by the equation
\begin{equation}\begin{split}\label{eq:bdlk}
 & \partial_t \rho_t = \nabla \cdot (\nabla \rho_t + \rho_t \nabla V) - \Lambda(x,\rho_t)\rho_t,\\
 & \text{ where }
 \Lambda(x,\rho_t) = \log(K*\rho_t) - \log \pi - \int_{\R^d} \log \left(\frac{(K*\rho_t)}{\pi}\right)\rho_t  \,dx.
\end{split}\end{equation}
For this equation, the solution $\rho_t$ can be approximated by the empirical measure $\mu^N_t$ of
a collection of interacting particles $\{x^i_t\}_{i=1}^N$ evolving as follows ($\mu^N_t = \frac{1}{N} \sum_{i=1}^N \delta_{x^i_t}$):

Step 1: between birth/death events, each particle $x^i$ diffuses independently according to \eqref{eq:oLan}.

Step 2: each particle also has an independent exponential clock
with instantaneous birth-death rate
\begin{align}
\Lambda(x^i_t) & = \log \Big(\frac{1}{N} \sum_{j=1}^N K(x^i_t - x^j_t)\Big) -  \log \pi(x^i_t) - \frac{1}{N}\sum_{\ell=1}^N
\Big(\log \Big(\frac{1}{N} \sum_{j=1}^N K(x^\ell_t - x^j_t)\Big) -  \log \pi(x^\ell_t)\Big)  \nonumber \\
& = \log( (K*\mu^N_t) (x^i_t))  -  \log \pi(x^i_t) - \int_{\R^d}  \log \left(\frac{K*\mu^N_t}{\pi} \right) \,d\mu^N_t. \label{eq:lambdaxi}
\end{align}
Specifically, if $\Lambda(x^i_t) > 0$, then partial $x^i$ is killed with instantaneous rate
$\Lambda(x^i_t)$ and another particle is duplicated randomly to preserve the population size; if
 $\Lambda(x^i_t) < 0$, then partial $x^i$ is duplicated with instantaneous rate
$|\Lambda(x^i_t)|$ and another particle is killed randomly to preserve the population size.
Thus the total number of particles is preserved. The
proposition below shows convergence of the empirical measure $\mu^N_t$ of
the particle system described above to the solution of \eqref{eq:bdlk} in the large particle limit.
Its proof can be found in Appendix \ref{sec:proofcov}.
\begin{prop}\label{prop:mfl}
  Let $\mu^N_t$ be the empirical measure of particles defined 
  above.
 Assume that $\mu^N_0 \wgt \rho_0$ as $N\gt\infty$. Then for all $t\in (0,\infty)$,
$\mu^N_t \wgt \rho_t$ where $\rho_t$ solves \eqref{eq:bdlk} with initial condition $\rho_0$.
\end{prop}

To implement the birth-death particle dynamics above in practice, we
also need time-discretization. In particular, discretizing the
Langevin diffusion by the Euler-Maruyama
scheme 
leads to the following birth-death accelerated Langevin sampler
(BDLS).

\begin{algorithm}[H]
 \caption{BDLS: birth-death accelerated Langevin sampler}\label{alg:ula-bd}
 \begin{flushleft}
        \textbf{Input:} A potential $V(x)$ corresponding to the target distribution $\pi(x)$, a set of initial particles $\{x^i_0\}_{i=1}^N$,
        number of iterations $J$, time step $\Delta t$, kernel function $K$\\
        \textbf{Output:} A set of particles $\{x^i_J\}_{i=1}^N$ whose empirical measure $\mu^N$ approximates $\pi$.
\end{flushleft}
 \textbf{for} $j = 1: J$ \textbf{ do }\\
 \hspace{1cm } \textbf{for} $i = 1: N$ \textbf{ do }\\
 \hspace{2cm} set $x^i_{j} = x^i_{j-1} - \Delta t  \nabla V(x^i_{j-1}) + \sqrt{2\Delta t} \xi_{j}$, where $\xi_j \stackrel{\textrm{i.i.d.}}{\sim} N(0,1)$ \\
 \hspace{2cm} calculate $\beta_i = \log\Big(\frac{1}{N}\sum_{\ell=1}^N K(x^i_j - x^\ell_j) \Big) + V(x^i_j)$ \\
 \hspace{2cm}  set $\bar{\beta}_i = \beta_i - \frac{1}{N}\sum_{\ell=1}^N \beta_\ell$\\
  \hspace{2cm} \textbf{if} $\bar{\beta}_i > 0$  \\
  \hspace{3cm} kill $x^i_j$ with probability $1 - \exp(-\bar{\beta}_i \Delta t)$,\\
 \hspace{3cm}  duplicate one particle that is uniformly chosen  from the rest\\
 \hspace{2cm} \textbf{else if} $\bar{\beta}_i < 0$\\
  \hspace{3cm} duplicate $x^i_j$ with probability $1 - \exp(\bar{\beta}_i \Delta t)$,\\
 \hspace{3cm}  kill one particle that is uniformly chosen  from the rest\\
 \hspace{2cm} \textbf{end if}\\
\hspace{1cm }\textbf{end for} \\

\textbf{end for } \end{algorithm}

\section{Numerical results}
In the numerical examples below, we compare the sampling efficiency of the proposed sampler BDLS of size $N$
 with the sampler built from running $N$ independent copies of ULA (we call it parallel ULA or simply ULA for short).
We choose the kernel $K$ to be the Gaussian kernel with width $h$, i.e. $K(x,y) = \frac{1}{(2\pi h^2)^{d/2}}\exp(-|x-y|^2/2h^2)$. The kernel width $h$ varies in
 different examples and is tuned to produce the best numerical performance. How to optimize the choice of $h$ with a sound theoretical basis is to be investigated in future work.
 \subsection{Example 1: multimodal distribution on a 1D torus.}  Consider
 the target $\pi(x) \propto \exp(-V(x))$ with
 $V(x) = 2.5\cos(2x) + 0.5\sin(x)$ defined on the torus
 $D:= [-2\pi, 2\pi]$. We initialize the continuous dynamics and
 particle systems according to the Gaussian distribution
 $\mathcal{N}(0,0.2)$. This makes sampling the measure $\pi$ difficult
 since $\pi$ has four modes on $D$, while $\rho_0$ is very peaked and
 almost does not overlap with $\pi$.  We show in the left figure of
 Figure \ref{fig:1d2} the convergence of KL-divergence
 $\mathrm{KL}(\rho_t | \pi)$, from which one sees that BDL-PDE
 \eqref{eq:fp3} substantially accelerates the slow convergence of FPE
 \eqref{eq:fp1} of Langevin diffusion, consistent with Theorem
 \ref{thm:converge2}.  The reason for fast convergence of BDE
 \eqref{eq:fp2} is unclear to us and will be investigated in the
 future.  To compare the particle algorithms, we plot in Figure
 \ref{fig:1d2} (the middle and right figures) the mean square errors
 (MSE) of BDLS and ULA in estimating the mean and variance of the
 target versus number of sample size.  We see that BDLS performs much
 better than ULA.  BDS performs the worst in this example (see
 snapshots in Figure \ref{fig:1d1}) as due to absence of diffusion the
 particles only rearrange themselves inside the small region around
 zero they initialize and never get out. Thus we do not plot the MSE
 of BDS as they are too large to be fitted in the same figure.  We
 choose the number of particles $N=100$, time step size
 $\Delta t = 0.03$ and a Gaussian kernel $K$ with width $h=0.05$ in
 this example. See Appendix \ref{sec:addex1} for more implementation details
 and additional numerical results.


\begin{figure}
  \centering
  \includegraphics[scale = 0.6]{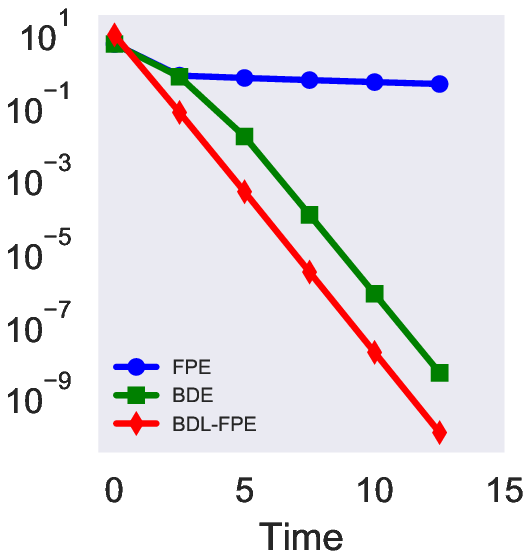}
 \includegraphics[scale = 0.6]{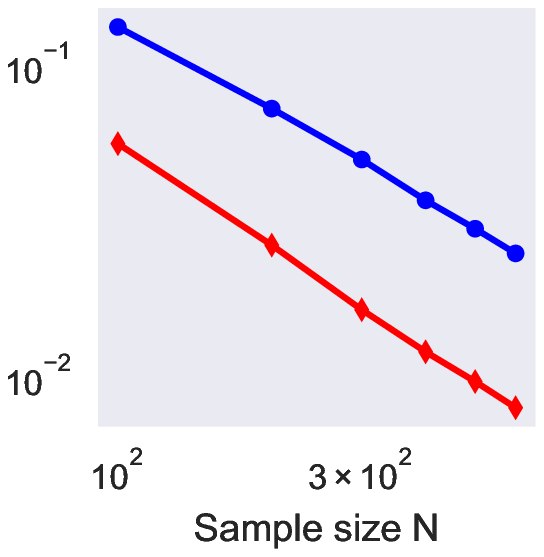}
 \includegraphics[scale = 0.6]{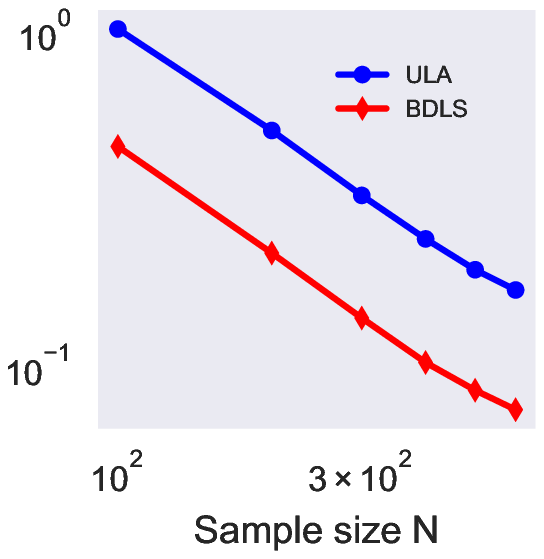}
\caption{Convergence of continuous dynamics and particles systems in Example 1. The left figure shows decay of the KL divergence in semilogy scale
along the evolution of three continuous dynamics. The middle (or the right) figure
shows the decay in loglog scale of mean square errors of estimating mean (or variance) using varying number of particles.}
\label{fig:1d2}\end{figure}

\subsection{Example 2: two dimensional Gaussian mixture.} Consider now a target of
two dimensional Gaussian mixture consisting of four  components
$
\pi(x,y) =\sum_{i=1}^4 w_i \mathcal{N}( m_i, \Sigma_i)
$
and initial particles
sampled from the Gaussian $\mathcal{N}(m_0, \Sigma_0)$, where
the parameters are defined by
\begin{equation*} \bea
& w_i = 1/4, i =1,\cdots,4,\ m_0 = m_1 = (0,8)^T,
m_2 = (0,2)^T,
m_3 = (-3,5)^T, m_4 = (3,5)^T,\\
& \Sigma_1 = \Sigma_2 = \begin{pmatrix}
              1.2 & 0\\
              0 & 0.01
             \end{pmatrix},
             \Sigma_3 = \Sigma_4 = \begin{pmatrix}
              0.01 & 0\\
              0&  2
             \end{pmatrix}, \Sigma_0 = \begin{pmatrix}
              0.3 & 0\\
              0 & 0.3
             \end{pmatrix}.
  \ena
\end{equation*}
In this example, we choose $N=10^3$ particles and use time step size $\Delta t = 10^{-3}$ for both ULA and BDLS algorithms.
Figure \ref{fig:gmm2d1} shows scatter plots along with  their corresponding marginals
of particles computed using parallel ULA  and BDLS  at different number of iterations.
The target distribution has a square shape and
the  particles are initialized within a small neighborhood of the top edge.
At the $10^4$-th iteration, the particles generated by BDLS already
start equilibrating around all modes, whereas only very few particles generated by parallel ULA reach to
the bottom mode at the same time.
We also compare the absolute error of estimating $\mathbb{E}_\pi[f]$ for different $f$ in  Figure \ref{fig:gmm2d2}.
We find that the estimation errors of using our BDLS converge to the lowerest values much faster than ULA.


\begin{figure}
  \raggedleft
  \includegraphics[scale=0.6]{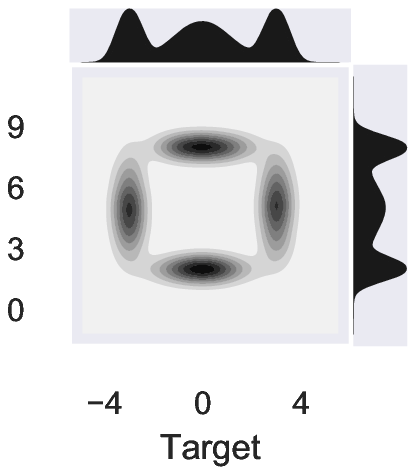}\hspace{-0.1in}
  \includegraphics[scale=0.6]{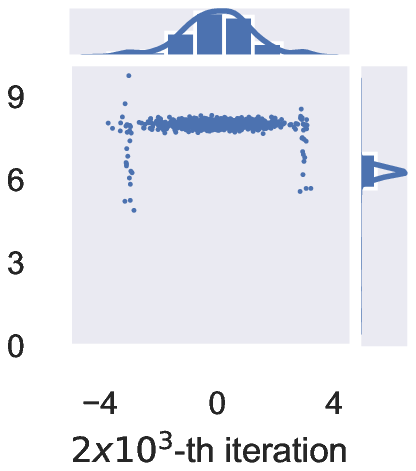}\hspace{-0.1in}
  \includegraphics[scale=0.6]{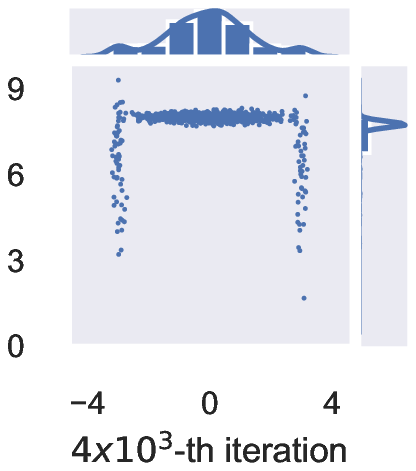}\hspace{-0.1in}
  \includegraphics[scale=0.6]{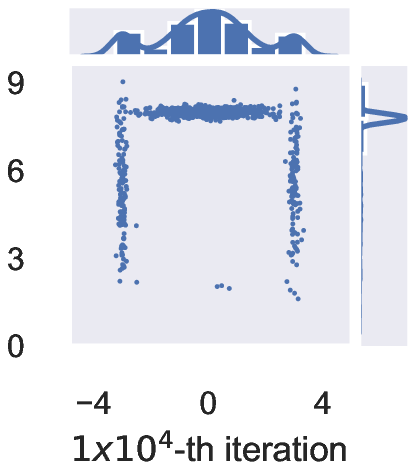}\hspace{-0.1in}
  \includegraphics[scale=0.6]{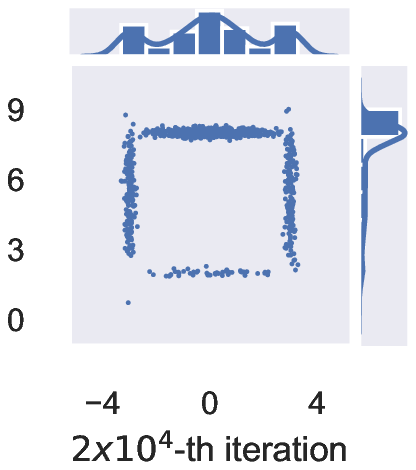}\hspace{-0.1in}

  \includegraphics[scale=0.6]{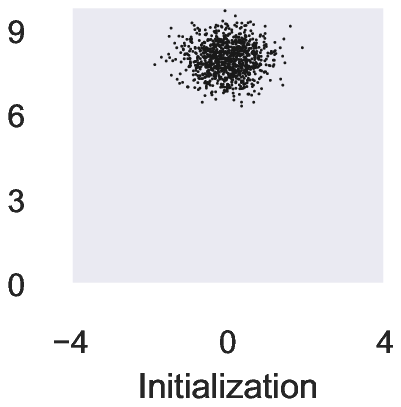}\hspace{-0.1in}
  \includegraphics[scale=0.6]{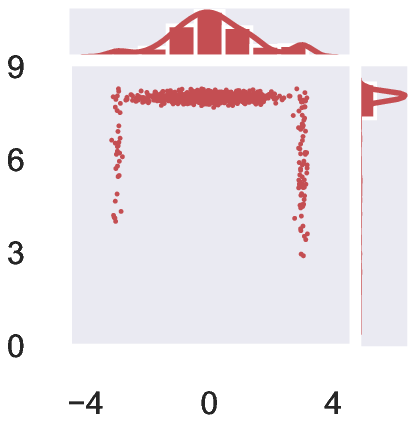}\hspace{-0.1in}
  \includegraphics[scale=0.6]{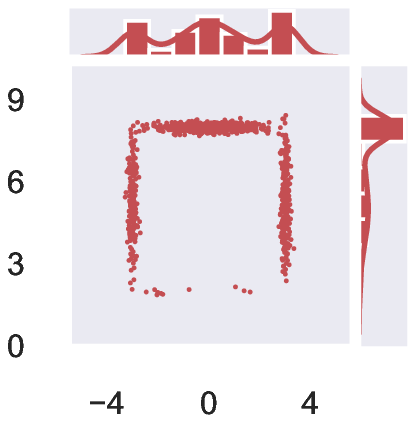}\hspace{-0.1in}
  \includegraphics[scale=0.6]{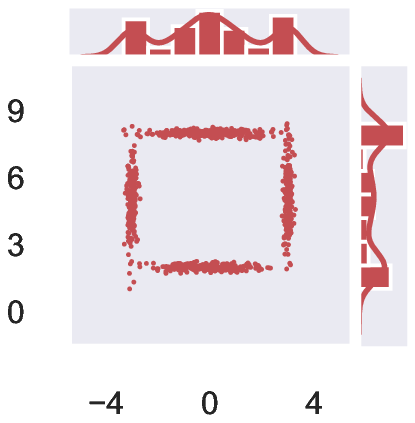}\hspace{-0.1in}
  \includegraphics[scale=0.6]{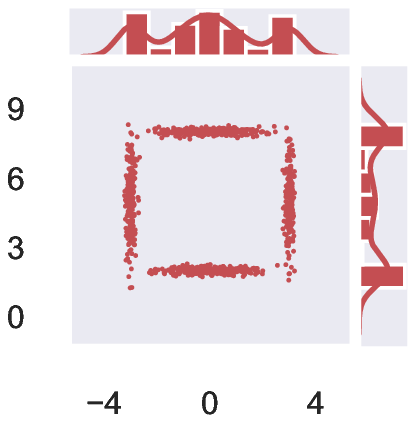}\hspace{-0.1in}
   \caption{Scatter plots of particles and their marginal distributions (computed by kernel density estimators) in Example 2.
   Top left figure displays the target density and the bottom left shows initial locations of particles.
   Each column in the rest shows the scatter plots and the marginal distributions
   of particles computed using parallel ULA (top, blue) and BDLS (bottom, red) at different iterations.}
    \label{fig:gmm2d1}
\end{figure}

\begin{figure}
  \centering
 \begin{subfigure}[t]{0.23\textwidth}
 \includegraphics[scale=0.8]{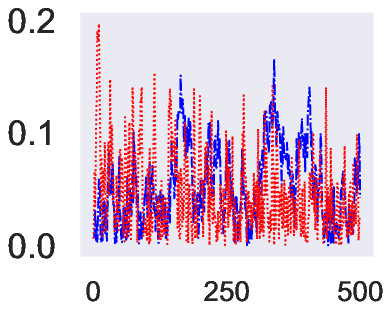}
  \caption{\tiny Estimating $\mathbb{E}[x]$}
 \end{subfigure}\hspace{0.05in}
  \begin{subfigure}[t]{0.23\textwidth}
  \includegraphics[scale=0.9]{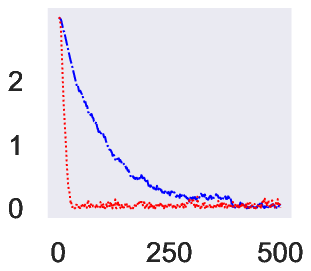}
  \caption{\tiny Estimating $\mathbb{E}[y]$}
   \end{subfigure}\hspace{-0.05in}
       \begin{subfigure}[t]{0.2\textwidth}
   \includegraphics[scale=0.9]{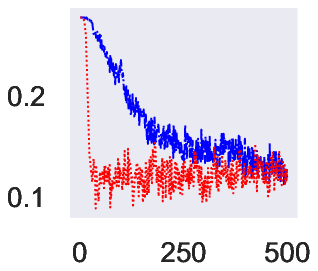}
     \caption{\tiny Estimating $\mathbb{E}[\chi(x,y)]$}
   \end{subfigure}
     \begin{subfigure}[t]{0.3\textwidth}
  \includegraphics[scale=0.9]{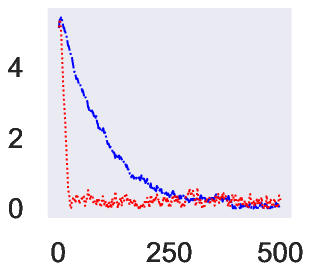}
  \caption{\tiny Estimating $\mathbb{E}[x^2/3 + y^2/5]$}
   \end{subfigure}\hspace{-0.05in}

   \caption{The absolute errors of estimating
   $\mathbb{E}[f(x,y)]$ with various observables $f$ in Example 2. In the third figure
   $\chi(x,y) = \mathbf{1}_{|x|\leq 5, |y-2|\leq 0.8}$. The blue dash-dot
   and red-dot lines
   are estimation errors along iterations using ULA and BDLS respectively. The total number of iterations is $2\times 10^{5}$. For the purpose of resolution, we plot the error for every 400 iterations.}
   \label{fig:gmm2d2}
\end{figure}

\subsection{Example 3: Bayesian learning of Gaussian mixture model.}
We consider the Bayesian approach to fitting the distribution of a dataset with
a univariate Gaussian mixture model of three components in the same setting as in  \cite{Chopin2012}.
The unknown parameters are the means $\mu_k$, precisions $\lambda_k$ and the weights
$w_k,k=1,2,3$ with $\sum_{k=1}^3 w_k =1$. We use the prior as in  \cite{Chopin2012} which has a hyperparamter $\beta$ describing the prior distribution of the precisions, thus
defining a posterior distribution $\pi$ on $\R^9$.
Due to the permutation invariance  with respect to the component label,
the resulting posterior has at least $3!= 6$ modes.
We generate a synthetic dataset of 200 samples from the mixture
measure with ``true'' parameters
$ w_1 = w_3=1/5, w_2 = 3/5, \mu_1 = -5,\mu_2 = 1, \mu_3 =6,
\lambda_k=1,k=1,2,3.  $ The data size is large enough to make the
posterior peaked so that hopping across different modes is
challenging.  We use $N=2000$ particles, time step size
$\Delta t = 1.5\times 10^{-6}$ and kernel width $h=1.1$. We initialize
particles as iid samples from the following distributions:
$(w_1,w_2)\in \text{Dirichlet}_3(1,1,1), \mu_k\sim \text{Unif}([3,7]),
\lambda_k \sim \text{Unif}([0.5,2.5]), \beta\sim
\text{Unif}([0.5,1.5])$.  To compare the performance of BDLS and that of ULA, we
show the evolution of sampling particles in $(\mu_1,\mu_2)$-coordinate
in Figure~\ref{fig:bayesgmm1} (see also Figure~\ref{fig:bayesgmm2} for
snapshots in $(w_1,w_2)$-coordinate). We see that BDLS algorithm
exhibits stronger mode exploration ability than ULA. Once all modes
are identified, BDLS quickly redistributes the particles in different
local modes towards the equilibrium through the birth-death process,
while ULA takes much longer time to equilibrate in the local modes.
In fact, the distribution of the BDLS particles in $(\mu_1,\mu_2)$ at
$2\times 10^4$-th iteration is already very close to the equilibrium
(see Figure~\ref{fig:bayesgmm3}). Appendix \ref{sec:addex3} contains further
details about the model and numerical results for this example.

\begin{figure}
  \raggedleft
  \includegraphics[scale=0.65]{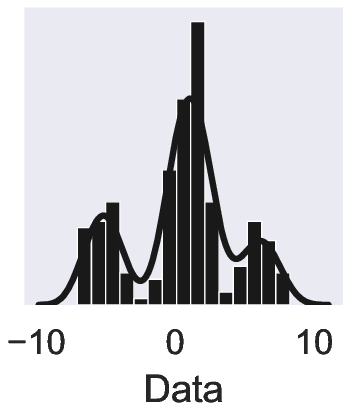}
  \includegraphics[scale=0.6]{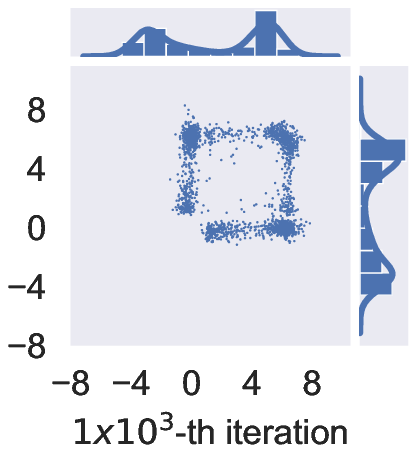}\hspace{-0.1in}
  \includegraphics[scale=0.6]{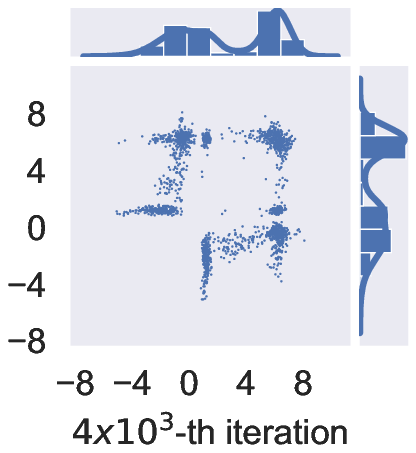}\hspace{-0.1in}
  \includegraphics[scale=0.6]{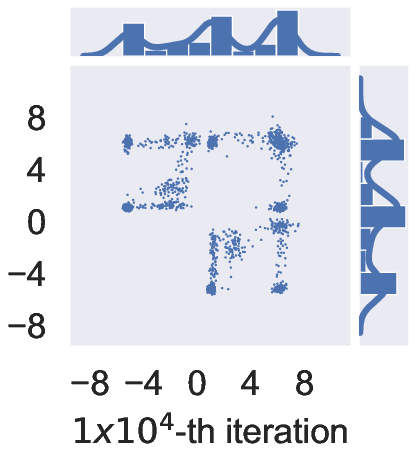}\hspace{-0.1in}
  \includegraphics[scale=0.6]{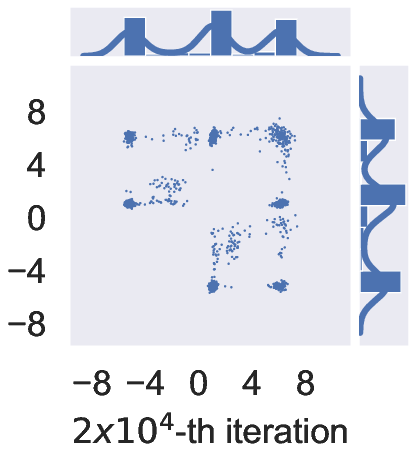}\hspace{-0.1in}
  \includegraphics[scale=0.6]{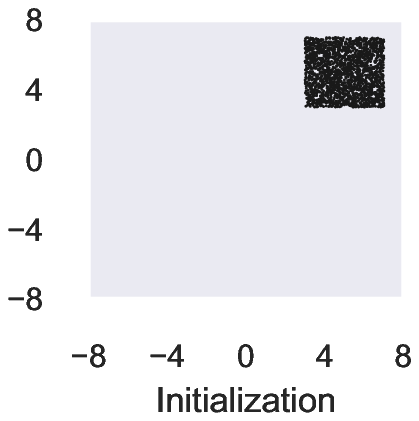}
  \includegraphics[scale=0.6]{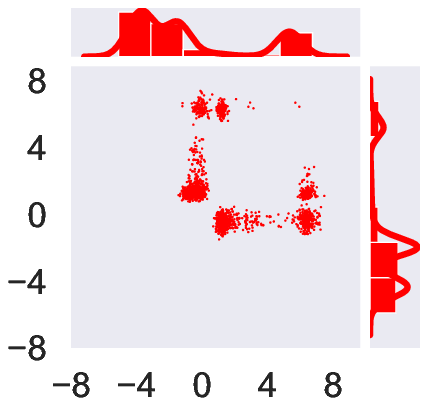}\hspace{-0.1in}
  \includegraphics[scale=0.6]{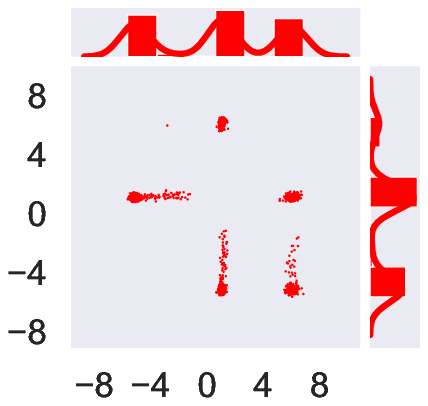}\hspace{-0.1in}
  \includegraphics[scale=0.6]{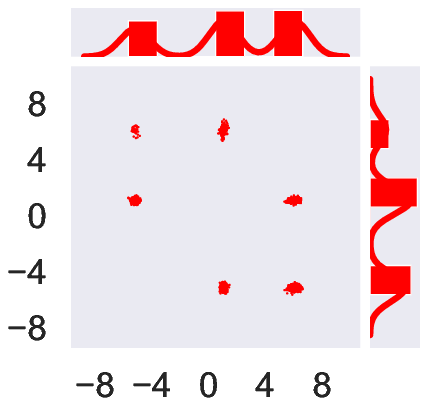}\hspace{-0.1in}
  \includegraphics[scale=0.6]{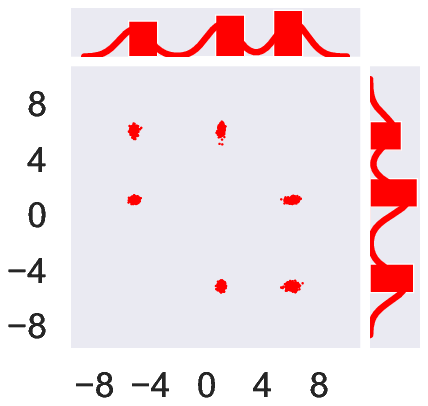}\hspace{-0.1in}
   \caption{Evolution of particles  in $(\mu_1,\mu_2)$-coordinate for  Example 3. The first column shows the histogram (top) of the synthetic data
   and the initial locations (bottom) of particles in $(\mu_1,\mu_2)$-coordinate.
   The rest columns compare the scatter plots of particles in $(\mu_1,\mu_2)$ and their marginals computed using parallel ULA (top, blue) and BDLS (bottom, red)
   at different iterations.}
 \vspace{-1em}
    \label{fig:bayesgmm1}
\end{figure}

\section{Conclusion}
We propose a new sampling dynamics based on birth-death process and an
algorithm based on interacting particles to accelerate the classical
Langevin dynamics for statistical sampling. Future directions include
a rigorous analysis of the birth-death accelerated Langevin sampler
and its further applications when combined with other conventional
sampling schemes.




\newpage

\appendix
\numberwithin{equation}{section}
\renewcommand\thefigure{\thesection.\arabic{figure}}

\section{Gradient flow structure}\label{sec:gradientflow}
This section devotes to the proof of Theorem \ref{thm:gf}.
 We mainly follow \cite{Otto01_porous} and \cite{GM17}.

We first introduce a Riemannian structure, denoted by $\M$
 on the space of smooth probability densities on $\R^d$.
Consider the tangent space at $\rho \in \M$
$$
\mathcal{T}_\rho \M := \Big\{ \text{ functions }\zeta \text{ on } \R^d \text{ satisfying } \int \zeta dx = 0 \Big\}.
$$
Since $\rho \geq 0$ is a probability density, the tangent space can also be identified as
$$
\mathcal{T}_\rho \M := \Big\{\zeta = - \nabla \cdot (\rho \nabla u )
+ \rho \Big(u - \int u \rho dx \Big)\Big\}.
$$
Indeed, there is ``one-to-one'' correspondence between $\zeta$ and $u $, since for any $\zeta$ such that $\int \zeta dx = 0$,
there exists $u \in H^1(d\rho)$ (determined uniquely up to a constant) solving
$$
\zeta = - \nabla \cdot (\rho \nabla u) + \rho \Big(u - \int u \rho dx \Big).
$$
Here  $H^1(d\rho)$ denotes the space of functions $u$ such that $\|u\|^2_{H^1(d\rho)} := \int (|\nabla u|^2 + |u|^2) d\rho < \infty$.
 Informed by the
Lagrangian minimization in the definition of the Wasserstein-Fisher-Rao distance \eqref{eq:dwfr},
we also define the Riemannian metric tensor $g_\rho(\cdot, \cdot): T_\rho \M \times T_\rho \M \gt \R$ by
\be\label{eq:Rie}
\bea
g_\rho( \zeta_1, \zeta_2 )  & := \int_{\R^d} \rho  \nabla u_1 \cdot \nabla u_2 dx + \int_{\R^d} \rho \Big(u_1 - \int u_1 \rho dx \Big)  \Big(u_2 - \int u_2 \rho dx \Big) dx\\
& =\int_{\R^d} \rho  \nabla u_1 \cdot \nabla u_2 dx + \int_{\R^d} u_1 u_2 \rho dx -  \int u_1 \rho dx \cdot \int u_2 \rho dx,
\ena
\en
where $\zeta_i = - \nabla \cdot (\rho \nabla u_i ) +\rho \Big(u_i - \int u_i \rho dx \Big), i=1,2$.
With this metric tensor $g_\rho$, the Wasserstein-Fisher-Rao distance defined by \eqref{eq:dwfr}  can be regarded as the geodesic distance on $\M$ with the  Riemannian metric $g_\rho$, namely,
\begin{equation}
\begin{split}
& d^2_{\text{WFR}} (\rho_0 , \rho_1) =\Big\{ \inf_{u_t}\int_0^1 g_{\rho_t}(\dot{\rho}_t, \dot{\rho}_t) dt\\
& \text{ s.t. } \dot{\rho}_t = - \nabla \cdot (\rho_t \nabla u_t ) + \rho_t \Big(u_t - \int_{\R^d} u_t \rho_t dx \Big)
,\quad \rho_t |_{t=0} = \rho_0, \quad   \rho_t |_{t=1} = \rho_1\Big\}.
\end{split}
\end{equation}

\begin{prop}\label{prop:gradF}
Let $\F : \M \gt \R$ be a continuous and differentiable energy functional. Then
the metric gradient of $\F (\rho)$  via the metric tensor $g_\rho$ is
\be\label{eq:gradF}
\text{grad}\, \F(\rho) = -\nabla\cdot \Big(\rho \frac{\delta \F(\rho)}{\delta \rho} \Big) + \rho\Big( \frac{\delta \F(\rho)}{\delta \rho} - \int \frac{\delta \F(\rho)}{\delta \rho} \rho dx\Big)
\en
As a result, the gradient flow of $\F (\rho)$
with respect to
the  Wasserstein-Fisher-Rao distance $d_{\text{WFR}}$ is given by
\be\label{eq:gradF2}
\begin{split}
\partial_t \rho & = -\text{grad}\, \F(\rho) \\
& = \nabla\cdot \Big(\rho \frac{\delta \F(\rho)}{\delta \rho} \Big) - \rho\Big( \frac{\delta \F(\rho)}{\delta \rho} - \int \frac{\delta \F(\rho)}{\delta \rho} \rho dx\Big).
\end{split}
\en
\end{prop}
\begin{proof}
  Let $\rho_t: t \gt \rho_t$ be a $C^1$ curve passing through
  $\rho_t |_{t=0} = \rho \in \P(\R^d)$ with tangent vector
  $$
  \frac{d \rho_t}{d t}|_{t=0} = \zeta =  - \nabla \cdot (\rho \nabla u) +\rho \Big(u - \int u \rho dx \Big).
  $$
  The gradient $\text{grad}\, \F$ with respect to $g_\rho(\cdot, \cdot)$ is defined by
  \be\label{eq:grad E}
  g_\rho( \text{grad}\, \F(\rho), \zeta )  = \frac{d \F(\rho_t)}{dt} |_{t=0}
  = \int \frac{\delta \F(\rho)}{\delta \rho} \zeta dx .
  \en
  By the  definition of the  Riemannian metric $g_\rho(\cdot, \cdot)$ in \eqref{eq:Rie}, the right hand side above is
    $$
  \bea
   \int \frac{\delta \F(\rho)}{\delta \rho} \zeta dx & = \int \frac{\delta \F(\rho)}{\delta \rho} \Big[ - \nabla \cdot (\rho \nabla \alpha) + \rho\Big(\alpha - \int \alpha \rho dx \Big)\Big] dx\\
  & = \int \rho \nabla \frac{\delta \F(\rho)}{\delta \rho}  \cdot \nabla \alpha  + \rho \Big(\frac{\delta \F(\rho)}{\delta \rho} - \int \frac{\delta \F(\rho)}{\delta \rho} \rho dx\Big) \alpha dx.\\
  & = \int  \rho \nabla \frac{\delta \F(\rho)}{\delta \rho}  \cdot \nabla \alpha  +  \rho \Big(\frac{\delta \F(\rho)}{\delta \rho} - \int \frac{\delta \F(\rho)}{\delta \rho} \rho dx\Big) \Big(\alpha - \int \alpha \rho dx \Big) dx.\\
 & = g_\rho\Big(-\nabla\cdot \Big(\rho \frac{\delta \F(\rho)}{\delta \rho} \Big) + \Big(\frac{\delta \F(\rho)}{\delta \rho} - \int \frac{\delta \F(\rho)}{\delta \rho} \rho dx\Big), \zeta\Big)
 \ena
  $$
  Since $\zeta$ is arbitrary, this proves \eqref{eq:gradF} and hence
  \eqref{eq:gradF2} follows.
\end{proof}

\begin{proof}[Proof of Theorem \ref{thm:gf}]
This is a direct consequence of Proposition \ref{prop:gradF} and the fact that the functional derivative of $\rho \mapsto \mathrm{KL}(\rho | \pi)$ is
$$
\frac{\delta \mathrm{KL}(\rho | \pi)}{\delta \rho} = \log \Big(\frac{\rho}{\pi}\Big) + 1.
$$
\end{proof}

\section{Proofs of convergence results}\label{sec:proofcov}
In this section we prove  Theorem \ref{thm:conv2}, Theorem \ref{thm:converge2}, Theorem \ref{thm:conv1}, and Proposition \ref{prop:mfl}.

\begin{proof}[Proof of Theorem \ref{thm:conv2}]
 Differentiating $\mathrm{KL}(\rho_t | \pi)$ in time gives
 \ben
 \frac{d}{dt} \mathrm{KL}(\rho_t | \pi)  = - \mathcal{I}(\rho_t | \pi)
 -\Big(\int \rho_t \big|\log \frac{\rho_t}{\pi}dx \big|^2 - \Big( \int \rho_t \log \frac{\rho_t}{\pi} dx\Big)^2 \Big).
 \enn
 Then the theorem is proved by using \eqref{eq:log-Sobolev} and the fact the second term on the right side above is non-positive due to the Cauchy-Schwartz inequality.
\end{proof}

\begin{proof}[Proof of Theorem \ref{thm:converge2}] It suffices to assume $t_0 = 0$.
First, we claim that if \eqref{fLower1} holds, then for all $t > 0$:
\begin{equation}
\inf_{x \in \mathbb{R}^d} \frac{\rho_t(x)}{\pi(x)} \geq  e^{- M e^{-t}}. \label{fLower2}
\end{equation}
This is because the function $\eta_t(x) = \log(\rho_t(x)/\pi(x))$ satisfies
\begin{equation}
\partial_t \eta = \Delta \eta + b(t,x) \cdot \nabla \eta  - \eta + \mathrm{KL}(\rho_t|\pi)  \geq \Delta \eta + b(t,x) \cdot \nabla \eta  - \eta \label{etapde1}
\end{equation}
where $b(t,x) = \nabla \log \rho_t(x)$.  By the maximum principle, the minimum of $\eta$, which must be negative, cannot decrease.  In fact, \eqref{fLower1} and \eqref{etapde1} implies $\eta_t(x) \geq e^{-t} \inf_{x} \eta_0(x) \geq - M e^{-t}$ so that $\rho_t(x)/\pi(x) \geq e^{- M e^{-t}}$, which is \eqref{fLower2}.  In particular, if $t \geq t_1 = |\log(\delta/M)|$, then we have
\begin{equation}
\inf_{x \in \mathbb{R}^d}\rho_t(x)/\pi(x) \geq e^{-\delta}. \label{fLower3}
\end{equation}

Now, under the evolution \eqref{eq:fp3}, the time derivative of $\mathrm{KL}(\rho_t | \pi)$ is
\be
 \frac{d}{dt} \mathrm{KL}(\rho_t | \pi)  = \underbrace{- \mathcal{I}(\rho_t | \pi)}_{\leq 0} - \int \rho_t   \Big|\log \Big(\frac{\rho_t}{\pi}\Big)\Big|^2 dx
      + \mathrm{KL}(\rho_t | \pi)^2. \label{eq:dkl}
\en
We may ignore the the first term on the right side since it is non-positive. Define $f_t = \frac{\rho_t}{\pi} - 1 \geq -1$. Then,
\[
\mathrm{KL}(\rho_t | \pi) = \int \rho_t \log\Big(\frac{\rho_t}{\pi}\Big) \,dx =\int \left((1 + f_t) \log(1 + f_t)  - f_t \right) \pi\,dx
\]
and
\[
\int \rho_t  \Big|\log \Big(\frac{\rho_t}{\pi}\Big)\Big|^2  dx = \int ( 1 + f_t) |\log(1 + f_t)|^2 \,\pi \,dx.
\]
Observe that the functions $H_1(f) = (1 + f)\log(1 + f) - f$ and $H_2(f) = (1 + f) |\log(1 + f)|^2$ are both non-negative for $f \geq - 1$, $H_1'(0) = H_2'(0) = 0$, and $H_1''(f) = 1/(1 + f)$ and $H_2''(f) = 2(\log(1 + f) + 1)/(1 + f) \geq (2 - 2\delta) H_1''(f)$ if $f \geq e^{-\delta} - 1$. Therefore, we have
\[
(2 - 2\delta) H_1(f) \leq H_2(f), \quad \quad \text{if}\;\; f \geq  e^{-\delta} - 1.
\]
The condition \eqref{fLower3} implies $\inf f_t(x) \geq e^{-\delta} - 1$ for all $t \geq t_1$. Combining these observations with \eqref{eq:dkl}, we see that
\begin{align}
 \frac{d}{dt} \mathrm{KL}(\rho_t | \pi) & \leq - \int H_2(f_t) \pi \,dx +  \mathrm{KL}(\rho_t | \pi)^2 \nonumber \\
& \leq - \int (2 - 2 \delta) H_1(f_t) \pi \,dx +  \mathrm{KL}(\rho_t | \pi)^2 \nonumber \\
& = - (2 - 2\delta) \mathrm{KL}(\rho_t | \pi)  + \mathrm{KL}(\rho_t | \pi)^2 \label{eq:dkl2}
\end{align}
holds for all $t \geq t_1$. Since $\mathrm{KL}(\rho_t | \pi) \leq \mathrm{KL}(\rho_0 | \pi) \leq 1$ also holds, by assumption, this implies
\be
 \frac{d}{dt} \mathrm{KL}(\rho_t | \pi)  \leq  - (2 - 2\delta) \mathrm{KL}(\rho_t | \pi)  +  \mathrm{KL}(\rho_t | \pi) \leq - \frac{1}{2} \mathrm{KL}(\rho_t | \pi)
\en
for all $t \geq t_1$, so that $\mathrm{KL}(\rho_t | \pi) \leq e^{- \frac{1}{2}(t- t_1)}\mathrm{KL}(\rho_0 | \pi) \leq e^{- \frac{1}{2}(t - t_1)}$.  Now, returning to \eqref{eq:dkl2}, we have
\be
 \frac{d}{dt} \mathrm{KL}(\rho_t | \pi)  \leq  - (2 - 2\delta) \mathrm{KL}(\rho_t | \pi) + e^{- \frac{1}{2}(t - t_1)} \mathrm{KL}(\rho_t | \pi)\leq- (2 - 3\delta) \mathrm{KL}(\rho_t | \pi)
\en
for all $t \geq t_1 + 2 |\log(\delta)| = \log(\frac{M}{\delta^3}) = t_*$.

\end{proof}

\begin{proof}[Proof of Theorem \ref{thm:conv1}]
If $\rho_t$ satisfies \eqref{eq:fp2}, then
 \begin{equation}
 \frac{d}{dt} \mathrm{KL}(\rho_t | \pi)  =  -\Big(\int \rho_t \big|\log \frac{\rho_t}{\pi}dx \big|^2 - \Big( \int \rho_t \log \frac{\rho_t}{\pi} dx\Big)^2 \Big) \leq 0.
 \end{equation}
So, $\mathrm{KL}(\rho_t | \pi)$ is non-increasing and hence is finite since $\mathrm{KL}(\rho_0 | \pi)$ is finite. By the monotone convergence theorem there exists $C_* \geq 0$ such that
 $\lim_{t\gt \infty}  \mathrm{KL}(\rho_t | \pi) = C_*$.

 Next we show that the solution $\rho_t(x) > 0$ if $\rho_0(x) > 0$.
 Let us denote $\eta_t(x) = \log (\rho_t(x)/\pi(x))$. Then $\eta_t$ solves the equation
 \be\label{eq:ht}
 \partial_t \eta_t = - \eta_t   +  \mathrm{KL}(\rho_t | \pi).
 \en
Hence, $\eta$ satisfies the relation
 \be\label{eq:ht2}
 \eta_t(x) = e^{-t} \eta_0(x) + \int_0^t e^{-(t-s)} \mathrm{KL}(\rho_s | \pi) ds, \quad \forall x\in \R^d, t > 0.
 \en
In particular, $t \mapsto \eta_t(x)$ is increasing if $\eta_t(x) < 0$,
which implies that $\eta_t(x) \geq \min(0,\eta_0(x)) > -\infty$.
As a result, $\rho_t(x) = e^{\eta_t(x)}\pi(x) > 0, \  \forall x\in \R^d$.
In addition, since $0 \leq \mathrm{KL}(\rho_t | \pi) \leq \mathrm{KL}(\rho_0 | \pi) <\infty$ and $\mathrm{KL}(\rho_t | \pi)\to C_*$, we have
 \be\label{eq:ht3}
 \lim_{t\gt \infty}  \int_0^t e^{-(t-s)} \mathrm{KL}(\rho_s | \pi)ds  = C_*.
 \en
Because of this and \eqref{eq:ht2}, we conclude that for any $x\in \R^d$, $\eta_t(x) \gt C_*$ as $t\gt\infty$, and $\rho_t(x) \gt e^{C_*}\pi(x)$ as $t\gt\infty$. However, since both $\pi$ and $\rho_t$ are probability densities, this implies $e^{C_*} = 1$, so that $C_* = 0$.  Thus, $\lim_{t\gt \infty} \rho_t(x)  = \pi(x)$ and $\mathrm{KL}(\rho_t | \pi) = C_* = 0$.
\end{proof}

\begin{proof}[Proof of Proposition \ref{prop:mfl}]
 We give a formal proof using the theory of measure-valued Markov process \cite{D93}; a similar proof strategy is used recently in \cite{Rotskoff2019neuron}.
 Our goal is to first derive the (infinite dimensional) generator and  the backward
 Kolmogorov equation of the measure-valued Markov process of $\mu^N_t = \frac{1}{N}\sum_{i=1}\delta_{x^i(t)}$.
To this end, let us define  for any smooth functional $\Psi : \P(\R^d) \gt \R$ the generator $\mathcal{L}_N$
$$
(\mathcal{L}_N \Psi) (\mu^N) := \lim_{t\downarrow0} \frac{\mathbb{E}_0 \Psi(\mu^N_t) - \Psi(\mu^N) }{t},
$$
where $\mathbb{E}_0 $ denotes the expectation of $\Psi(\mu^N_t) $ conditioned on that $\mu^N_0  = \mu^N$.
To evaluate the limit above, notice that by definition
whenever a particle $x^i_t$ is killed (or duplicated) at time $t$,
another particle $x^j_t$ is duplicated (or killed) instantaneously. The birth or death is dictated by
the sign of the birth-death rate $\Lambda (x^i)$ defined by \eqref{eq:lambdaxi}.
As a result, the instantaneous change from $\mu_{t_-}^N$ to $\mu_t^N$ is
$$
\mu_t^N - \mu_{t_-}^N = -\frac{1}{N} \text{sign}(\Lambda(x^i)) (\delta_{x^i} - \delta_{x^j}).
$$
It is thus useful to define the empirical measure after a swap happens between $x$ and $x^\prime$ at time $t$ by
\begin{equation}\label{eq:mutnx}
 \mu_t^{N}\{x \leftrightarrow x^\prime\} = \mu_t^{N} - \frac{1}{N} \text{sign}(\Lambda(x)) (\delta_{x} - \delta_{x^\prime}).
\end{equation}
Since the particles are undergoing Langevin diffusions independently before a swap occurs between $x^i$ and $x^j$ occurs with an exponential rate $\Lambda(x^i)$, we can derive that
\begin{equation}
\begin{split}
 (\mathcal{L}_N \Psi) (\mu^N) & = \frac{1}{N} \sum_{i=1}^N \int \Big(\nabla_x D_{\mu^N}\Psi(x^i) \cdot \nabla \log \pi (x^i) + \Delta_x D_{\mu^N} \Psi(x^i)\Big) \delta_{x^i} (dx)
 \\
& + \frac{1}{N} \sum_{i,j=1}^N \iint |\Lambda(x^i)| \delta_{x^i}(dx) \delta_{x^j}(dx)  \Big(\Psi(\mu^{N}\{x^i \leftrightarrow x^j\} - \Psi(\mu^N)\Big),
\end{split}
\end{equation}
where the functional derivative  $D_\mu \Psi (x)$ is a function from $\R^d \gt \R$ defined by that for any signed measure $\nu$ with $\int \nu(dx) =0$,
\begin{equation}\label{eq:Dmu}
 \lim_{\eps \gt 0} \frac{\Psi(\mu + \eps \nu) - \Psi(\nu)}{\eps}=  \int D_\mu \Psi(x) \nu (dx).
\end{equation}
Now by the definition of the empirical measure $\mu^N$, the generator $\mathcal{L}_N$ can be rewritten as
\begin{equation}\label{eq:Ln}
\begin{split}
(\mathcal{L}_N \Psi) (\mu^N) & = \int \Big(\nabla_x D_{\mu^N}\Psi(x) \cdot \nabla \log \pi (x) + \Delta_x D_{\mu^N} \Psi(x)\Big) \mu^N(dx) \\
& + N \iint |\Lambda (x, \mu^N)| \mu^N(dx)\mu^N(dx^\prime)
\Big(\Psi(\mu^N \{x\leftrightarrow x^\prime \}) - \Psi (\mu^N)\Big),
\end{split}
\end{equation}
where $\Lambda (x, \mu)$ is defined by
$$
\Lambda (x, \mu) = \log (K\ast \mu (x)) - \log \pi (x) - \int \Big( \log (K\ast \mu (x)) - \log \pi (x)\Big) \mu(dx).
$$
Note that the measure $\mu^N \{x\leftrightarrow x^\prime \}$ on the right side of \eqref{eq:Ln}
is defined in \eqref{eq:mutnx} with jump rate $\Lambda(x,\mu^N)$.
With the generator, we can write the backward Kolmogorov equation on the observable $\Psi(\mu^N_t)$ as
$$
\partial_t \Psi(\mu^N_t) = \mathcal{L}_N \Psi(\mu^N_t), \quad \Psi(\mu^N_t)|_{t=0} = \Psi(\mu^N_0).
$$
Now passing to the limit $N\gt\infty$ and assuming that $\mu^N_t \gt \rho_t$
we claim that formally we have  $ (\mathcal{L_N} \Psi)(\mu^N_t) \gt (\mathcal{L} \Psi)(\rho_t)$
where the limiting generator $\mathcal{L}$ is given by
\begin{equation}\label{eq:L}
\begin{split}
 (\mathcal{L} \Psi)(\mu) =  \int \Big(\nabla_x D_{\mu}\Psi(x) \cdot \nabla \log \pi (x)+ \Delta_x D_{\mu} \Psi(x) \Big) \mu(dx)
 - \int \Lambda (x, \mu) \mu(dx)D_{\mu}\Psi(x).
\end{split}
\end{equation}
In fact, by assumption the first term on the right side of \eqref{eq:L} is the formal limit of the first term on the right side of \eqref{eq:Ln}.
For the second term, one first sees from the definition of the functional derivative in \eqref{eq:Dmu} that as $N\gt\infty$
$$
\Psi(\mu^N \{x\leftrightarrow x^\prime \}) - \Psi (\mu^N) \approx -\frac{1}{N}\int D_{\mu}\Psi(y) \text{sign} (\Lambda(x,\mu^N)) (\delta_x(dy) - \delta_{x^\prime}(dy)).
$$
This implies that as $N\gt\infty$ the second term on the right side of \eqref{eq:L} formally converges to
\begin{equation}\label{eq:L2}
\begin{split}
& -\iint|\Lambda (x, \mu)| \mu(dx)\mu(dx^\prime)
\int D_{\mu}\Psi(y) \text{sign} (\Lambda(x,\mu)) (\delta_x(dy) - \delta_{x^\prime}(dy))\\
&= -\iint\int D_{\mu}\Psi(y) \Lambda(x,\mu) \delta_x(dy) \mu(dx) \mu(dx^\prime)  + \iint\int D_{\mu}\Psi(y) \Lambda(x,\mu) \delta_{x^\prime}(dy) \mu(dx) \mu(dx^\prime) \\
& = -\int D_{\mu}\Psi(x) \Lambda(x,\mu)\mu(dx),
\end{split}
\end{equation}
where the last line follows from the fact that
\begin{equation*}
\begin{split}
  \iint\int D_{\mu}\Psi(y) \Lambda(x,\mu)\delta_{x^\prime}(dy) \mu(dx) \mu(dx^\prime)  = \int  D_{\mu}\Psi(y) \mu(dy)  \cdot
  \underbrace{\int \Lambda(x,\mu)\mu(dx)}_{=0}
  =0
\end{split}
  \end{equation*}
Combining above yields \eqref{eq:L}. Consequently, we obtain
the mean field backward Kolmogorov equation
$$
\partial_t \Psi(\rho_t) = (\mathcal{L} \Psi) (\rho_t), \quad  \Psi(\rho_t) |_{t=0} = \Psi(\rho_0).
$$
It is easy to check that this equation is precisely the time-evolution of $\Psi (\rho_t)$ where $\rho_t$ solves \eqref{eq:bdlk}.
This shows that $\mu^N_t \wgt \rho_t$ and concludes the proof.
\end{proof}

\section{More details on Example 1} \label{sec:addex1}
Let us first explain how we compute the numerical solutions of three continuous dynamics.
The Fokker-Planck equation \eqref{eq:fp1} is solved using the pseudo-spectral discretization in space and an
implicit backward Euler discretization in time. The pure birth-death equation \eqref{eq:fp2} is solved approximately by using the splitting scheme of alternating the following two steps:

Step 1: evolve the ODE system $\frac{d\rho_{t}(x)}{d t}= -\rho_{t} (\log \rho_t(x) - \log \pi(x))$ indexed by $x$ for a small time step $\Delta t$.

Step 2: renormalize the solution by setting $\rho_t(x) \leftarrow \rho_t(x)/\int \rho_t(x) dx$.

When $\Delta t$ is sufficiently small, this splitting scheme provides a good approximation to \eqref{eq:fp2}.
The Fokker-Planck equation with birth-death \eqref{eq:fp3} is solved by first evolving the Fokker-Planck equation \eqref{eq:fp1}
for a time step $\Delta t$ using the pseudo-spectral method and
then evolving the birth-death equation \eqref{eq:fp2} using the splitting scheme above for another time step $\Delta t $. We use $500$ spatial grids points in pseudo-spectral method
and time step size $5\times10^{-3}$ in time-marching.

We show in Figure \ref{fig:1d1} some snapshots of solutions of three continuous dynamics and the corresponding particle algorithms for Example 1,
which illustrates the acceleration effect of the birth-death dynamics on the Langevin dynamics.

\begin{figure}
  \centering
   \includegraphics[]{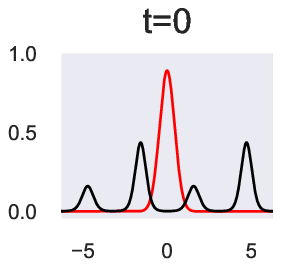}
    \includegraphics[]{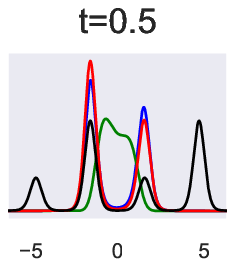}
    \includegraphics[]{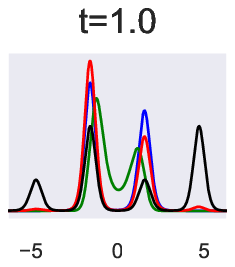}
        \includegraphics[]{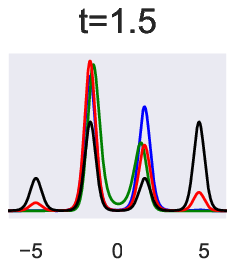}
            \includegraphics[]{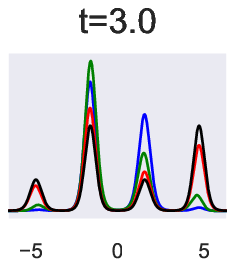}
            \includegraphics[scale=0.6]{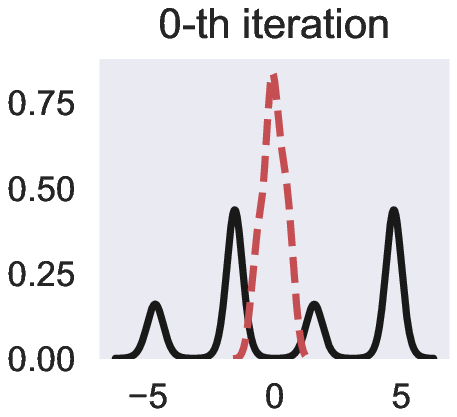} %
       \includegraphics[scale=0.6]{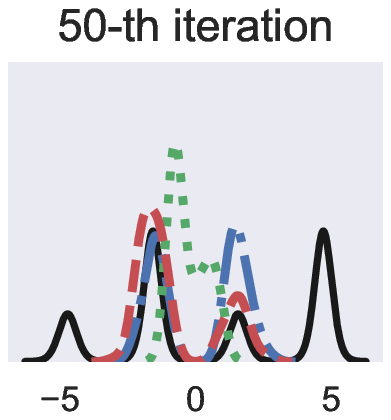}
      \includegraphics[scale=0.6]{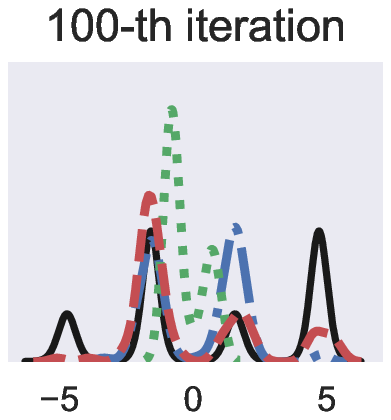}
 \includegraphics[scale=0.6]{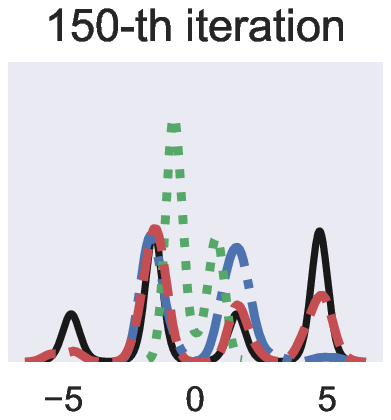}
    \includegraphics[scale=0.6]{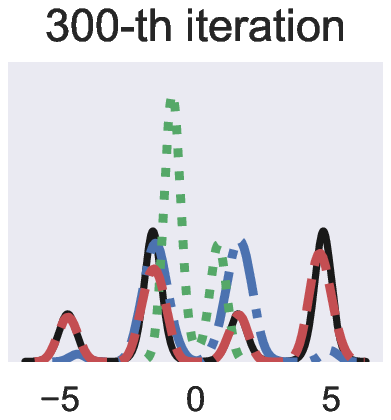}
     \vspace{-0.2in}
  \caption{Solutions of continuous dynamics (top row) at varying times and  distributions (kernel density estimators) of the corresponding particle algorithms (bottom row)
  at different iterations in Example 1. The initial distribution is $\mathcal{N}(0,0.2)$.
  The solid black lines are the target density and the blue (resp., blue dash-dot) lines are solutions of the FPE
  (resp.,  iterates of parallel ULA). The green and green dotted lines are solutions of BDE and the distributions of particles computed using BDS respectively.
  The red lines and red dashed lines are solutions of BDL-FPE and the distributions of particles computed using BDLS respectively. }
  \label{fig:1d1}\end{figure}

We present another group of numerical results for Example 1 in Figure \ref{ps1dadd1} and Figure \ref{ps1dadd2},
where we choose a Gaussian initial distribution with a larger variance $\sigma =4$. As before, we find that our algorithm BDLS outperforms ULA.
Observe that the particle algorithm BDS (based on pure birth-death dynamics)
works also well in this case because the initial particles are spread out on the whole domain so that they can quickly cluster around
different modes by rearranging their locations.

\begin{figure}
\centering
   \includegraphics[]{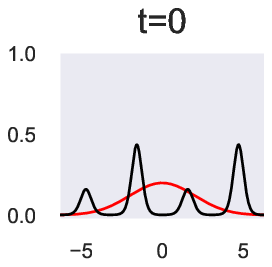}
    \includegraphics[]{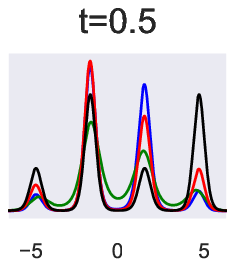}
    \includegraphics[]{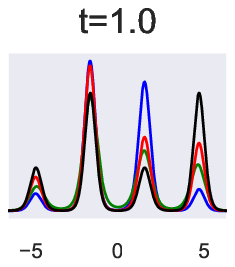}
        \includegraphics[]{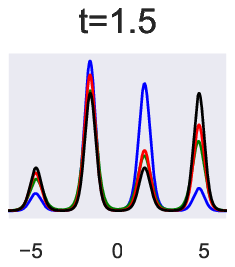}
            \includegraphics[]{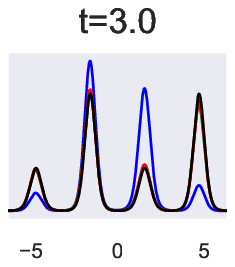}
                        \includegraphics[scale=0.6]{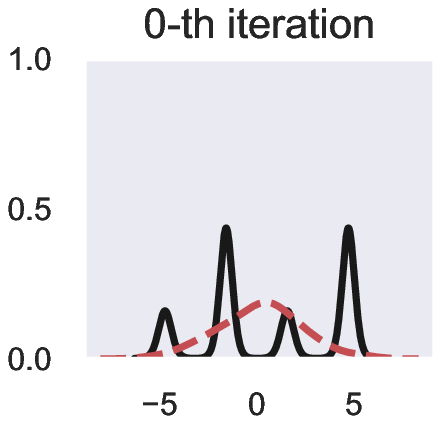} %
       \includegraphics[scale=0.6]{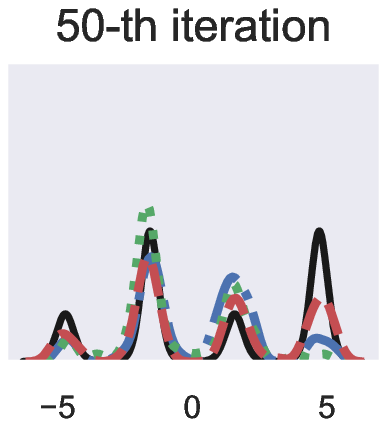}
      \includegraphics[scale=0.6]{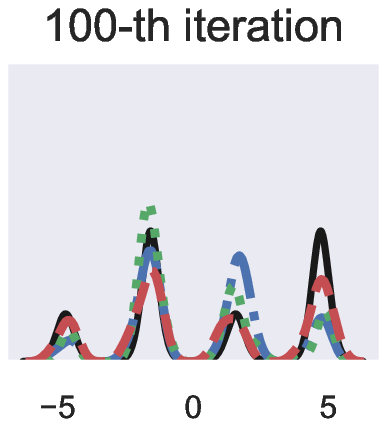}
 \includegraphics[scale=0.6]{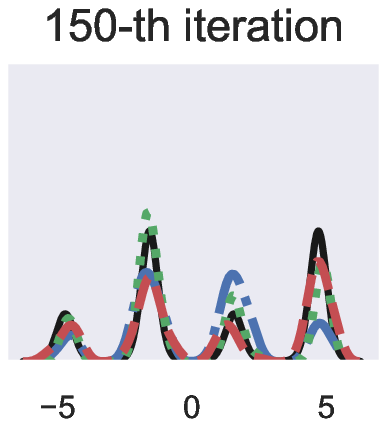}
    \includegraphics[scale=0.6]{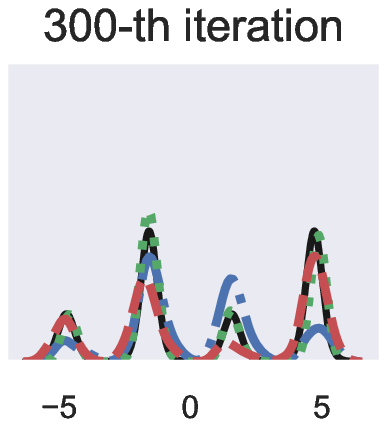}
    \caption{The same setting as in Figure \ref{fig:1d1} for Example 1 but with initial distribution $\mathcal{N}(0,4)$.
  The solid black lines are the target density and the blue (resp., blue dash-dot) lines are solutions of the FPE
  (resp.,  iterates of parallel ULA). The green and green dotted lines are solutions of BDE and the particles generated using BDS respectively.
  The red lines and red dashed lines are solutions of BDL-FPE and particles generated using BDLS respectively. }
  \label{ps1dadd1}
\end{figure}

\begin{figure}
  \centering
\includegraphics[scale = 0.6]{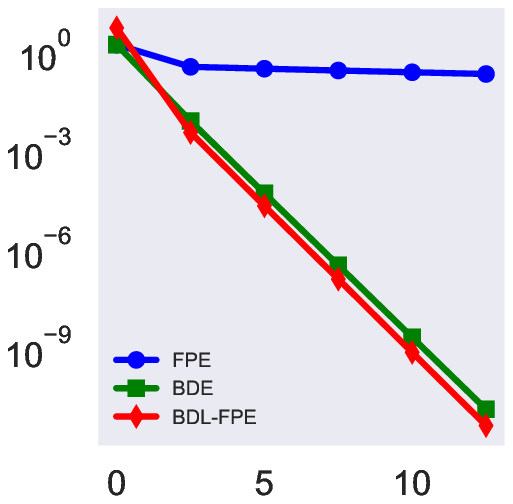}
 \includegraphics[scale = 0.6]{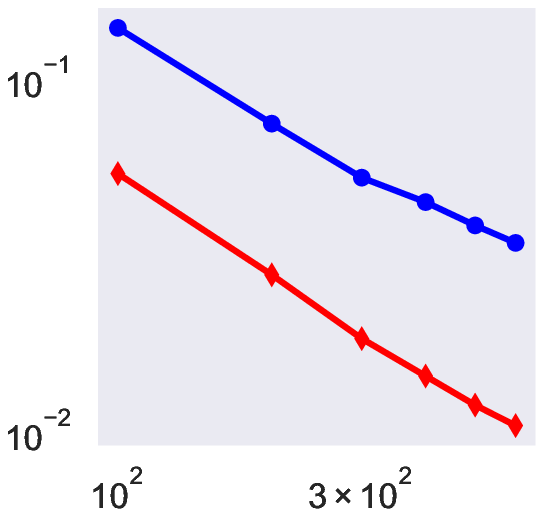}
 \includegraphics[scale = 0.6]{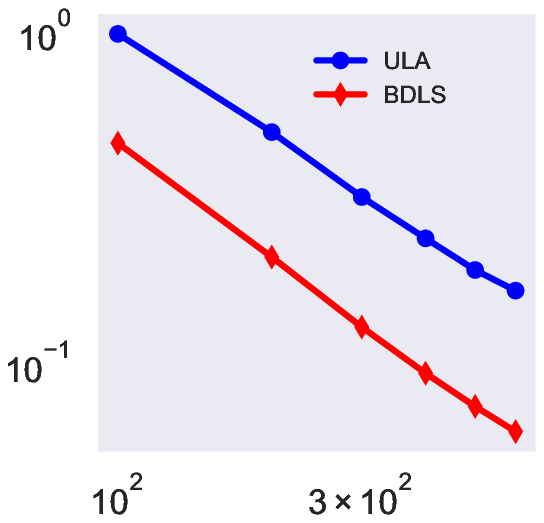}
\caption{The same setting as in Figure \ref{fig:1d2} for Example 1 but with initial distribution $\mathcal{N}(0,4)$.
The left figure shows decay of the KL divergence in semilogy scale
along the evolution of three continuous dynamics. The middle (or the right) figure
shows the decay in loglog scale of mean square errors of estimating mean (or variance) using varying number of particles.}
  \label{ps1dadd2}
  \end{figure}


\section{More details on Example 3}\label{sec:addex3}
We provide more details on the Bayesian Gaussian mixture model used in Example 3. Let $y = \{y_1,\cdots, y_n\}$ be a dataset consisting of an i.i.d. sequence of samples from
 the Gaussian mixture distribution
$$
p(y | x) = \sum_{k=1}^3 w_k \mathcal{N}(y; \mu_k, \lambda_k^{-1})
$$
where $\mu_k$ and $\lambda_k \geq 0$ are the means and precisions of the Gaussian components. The weights $\{w_k\}_{k=1}^3$ satisfy that $0\leq w_k \leq 1$ and
that $\sum_{k=1}^3 w_k = 1$. We denote by $x$ the vector of parameters/hyperparamters in this model. We take the same prior distribution as in \cite{Chopin2012} and \cite{LMW18}, namely for $k=1,2,3$,
$$
\mu_k \sim \mathcal{N}(m, \kappa^{-1}),\ \lambda_k \sim \text{Gamma}(\alpha, \beta), \ \beta \sim  \text{Gamma}(g, h),\ (w_1, w_2) \sim  \text{Dirichlet}_3(1,1,1).
$$
We also choose $m = M, \kappa=4/R^2, \alpha = 2, g = 0.02, h = 100 g/ (\alpha R^2)$, where $R$ and $M$ are the mean and range of the data $y$. By the Bayes' rule,
the posterior is given by
\begin{equation}
 \begin{split}
p(x|y)& \propto \beta^{3\alpha + g -1} \Big(\prod_{k=1}^3\lambda_k\Big)^{\alpha-1} \exp\Big(-\frac{\kappa}{2}\sum_{k=1}^3 (\mu_k - m)^2 -\beta \big(h + \sum_{k=1}^3 \lambda_k\big)\Big)\\
& \times \prod_{i=1}^n\Big( \sum_{k=1}^3 w_k \lambda_k^{1/2}\exp\big(-\frac{\lambda_k}{2} (y_i - \mu_k)^2\big)\Big)
\end{split}
\end{equation}
The unknown  vector $x$ of  parameters  is
$$
x = (w_1,w_2,\mu_1,\mu_2,\mu_3, \lambda_1,\lambda_2,\lambda_3,\beta) \in \Omega := \mathcal{S}_3 \times \R^3 \times \R_+^4,
$$
where $\R_+ = [0,\infty)$ and $\mathcal{S}_3 = \{(w_1,w_2)\in \R_+^2 |\ 0\leq w_1 + w_2 \leq 1\}$ is the probability simplex in $\R^2$.
There are several issues in the implementation of ULA and BDLS. First, the constraints on $w_k, \lambda_k$ and $\beta$
may be violated if the vanilla ULA and BDLS are applied on the whole space without additional treatment during the evolution. Note also that the posterior
density is not differentiable near the boundary of $\Omega$.
Moreover, even inside the domain $\Omega$, the gradient of $\log \pi$ may not be globally Lipschitz, which may lead to non-ergodic Markov chains
when applying ULA and BDLS. To overcome the latter issue, we use the following tamed ULA scheme \cite{hutzenthaler2012}
$$
x_{k+1} = x_{k} +  \frac{\Delta t \nabla \log \pi (x_{k})}{1 + \Delta t |\nabla \log \pi (x_{k})|} + \sqrt{2 \Delta t} \xi_k, k=1,2,\cdots,
$$
where $\xi_k \sim \mathcal{N}(0,1)$. The small modification of the drift stabalizes the algorithm and makes the resulting Markov chain ergodic;
see \cite{hutzenthaler2012,BDMS18} for more discussions about its convergence analysis.

To circumvent the constraint issue, we set a reflecting boundary at the origin for the
parameters $\lambda_k$ and $\beta$, i.e. we take the modulus
of these parameters if they become negative. For the weight vector $(w_1,w_2)$,  to improve sampling efficiency we
slightly relax the strong constraint that $(w_1,w_2)\in \mathcal{S}_3$ and instead only
require that $0\leq w_k \leq 1,k=1,2$.  We achieve this by setting reflections on the boundary $w=0$ and $w=1$.
Numerical experiments show that the relaxation does not break this constraint on the samplers near equilibrium;
see Figure \ref{fig:bayesgmm3b}.

Finally we include several numerical results on Example 3  that are not fitted in the main paper. Figure \ref{fig:bayesgmm2} compares the
evolution of particles computed using parallel ULA and BDLS in $(w_1,w_2)$.

\begin{figure}
  \raggedleft
    \hspace{1in}
  \includegraphics[scale=0.6]{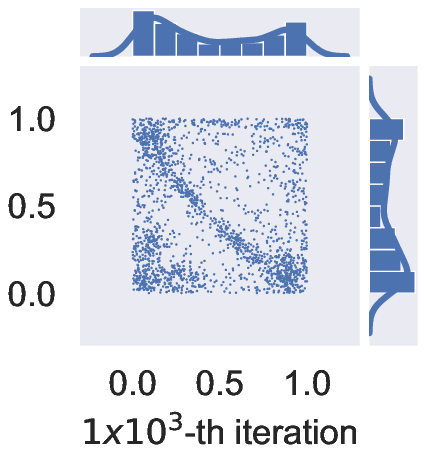}\hspace{-0.1in}
  \includegraphics[scale=0.6]{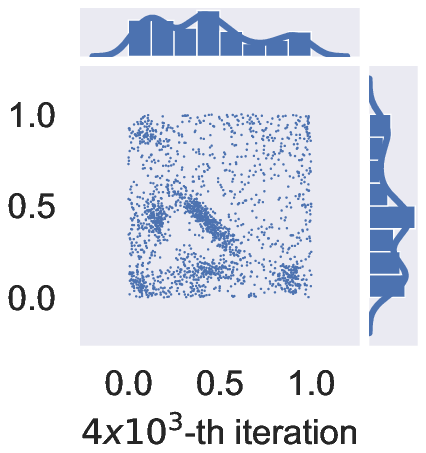}\hspace{-0.1in}
  \includegraphics[scale=0.6]{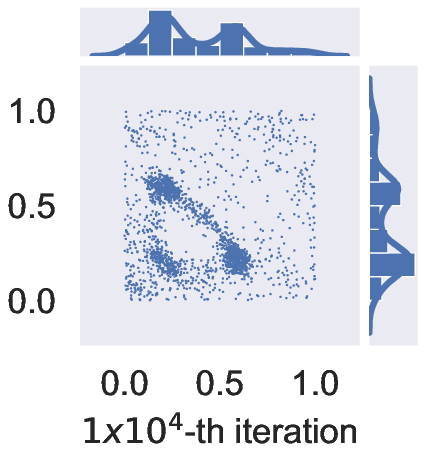}\hspace{-0.1in}
  \includegraphics[scale=0.6]{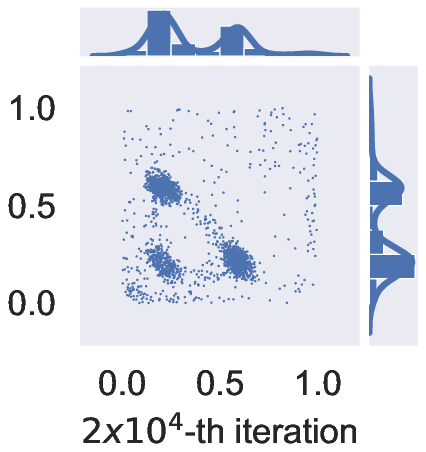}\hspace{-0.1in}

  \includegraphics[scale=0.6]{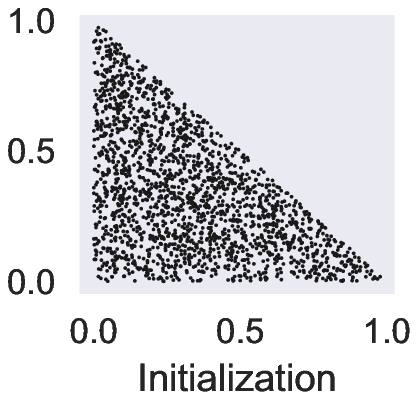}
  \includegraphics[scale=0.6]{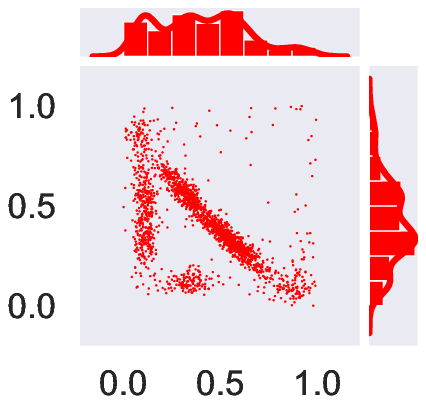}\hspace{-0.1in}
  \includegraphics[scale=0.6]{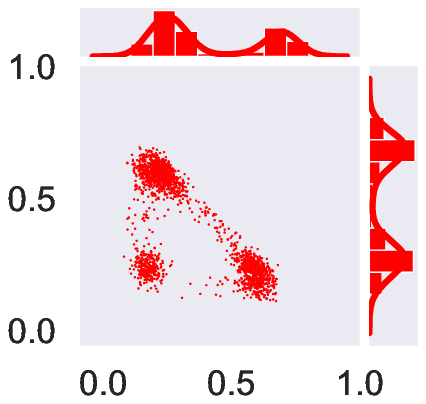}\hspace{-0.1in}
  \includegraphics[scale=0.6]{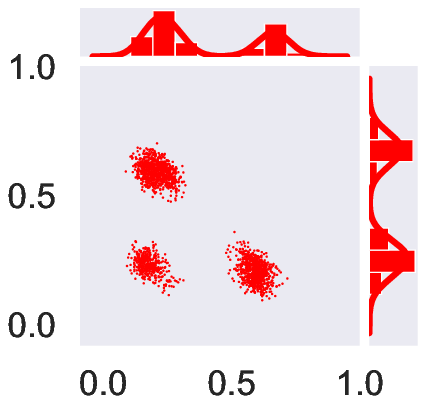}\hspace{-0.1in}
  \includegraphics[scale=0.6]{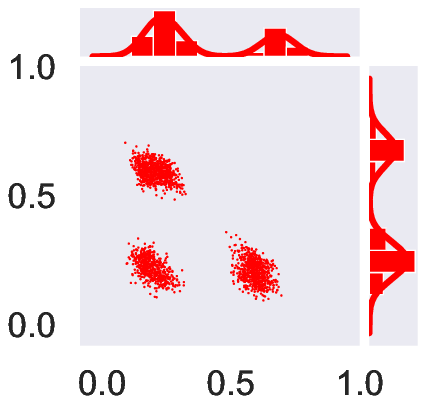}\hspace{-0.1in}
   \caption{Evolution of particles   in $(w_1,w_2)$-coordinate for  Example 3. The first column shows the initial  particles in $(w_1,w_2)$-coordinate. The remaining columns compare the scatter plots of particles in $(w_1,w_2)$ and their marginals computed using parallel ULA (blue) and BDLS (red) at different iterations. The constraint $0\leq w_1 +w_2\leq 1$ is relaxed to $0\leq w_k\leq 1$ in implementing ULA and BDLS to boost sampling efficiency.}
    \label{fig:bayesgmm2}
\end{figure}

Figure \ref{fig:bayesgmm3} displays the distribution
of particles in $(\mu_1,\mu_2)$ and in $(w_1,w_2)$ at larger numbers of iterations (near equilibrium), which complements
Figure \ref{fig:bayesgmm1} and  Figure \ref{fig:bayesgmm2} in illustrating the faster convergence of BDLS compared to ULA.

\begin{figure}
   \begin{subfigure}[t]{0.4\textwidth}
       \centering
  \includegraphics[scale=0.6]{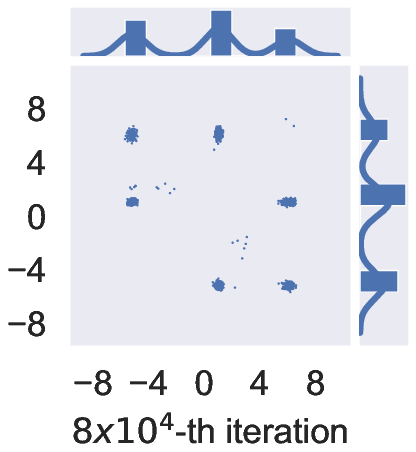}
  \includegraphics[scale=0.6]{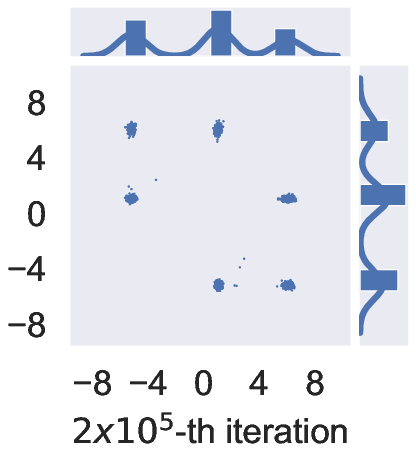}
   \includegraphics[scale=0.6]{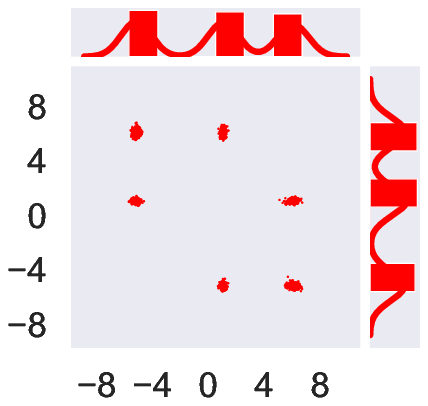}
  \includegraphics[scale=0.6]{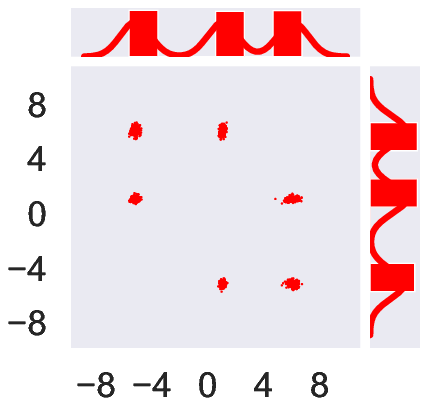}
  \caption{$(\mu_1,\mu_2)$-coordinate}\label{fig:bayesgmm3a}
  \end{subfigure}
    \begin{subfigure}[t]{0.4\textwidth}  \centering
        \includegraphics[scale=0.6]{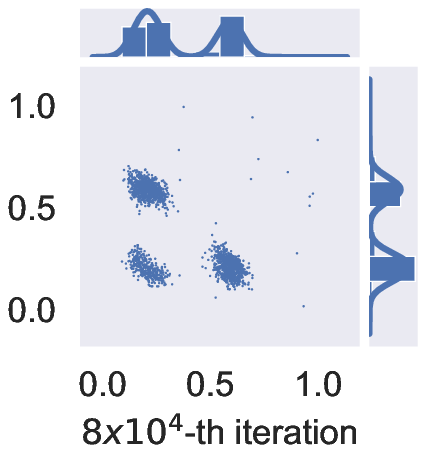}
  \includegraphics[scale=0.6]{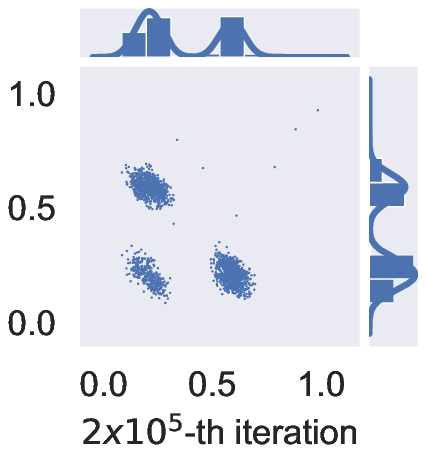}
  \includegraphics[scale=0.6]{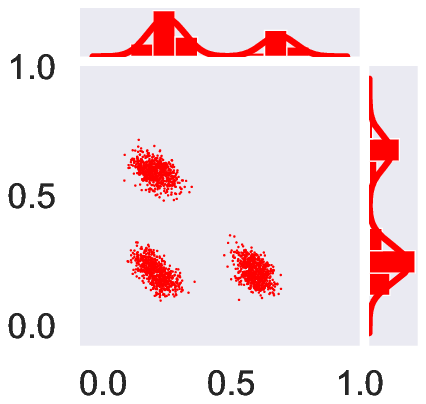}
  \includegraphics[scale=0.6]{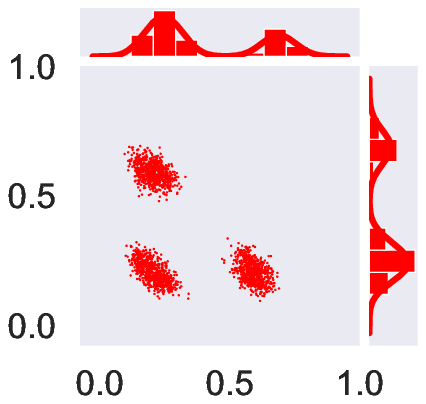}
    \caption{$(w_1,w_2)$-coordinate}\label{fig:bayesgmm3b}
    \end{subfigure}
   \caption{Additional snapshots of particle evolution in $(\mu_1,\mu_2)$-coordinate and $(w_1,w_2)$-coordinate for Example 3.
   As in Figure \ref{fig:bayesgmm1}, the top row (blue) and the bottom row (red) show the scatter plots of particles and their marginals computed using parallel ULA  and BDLS respectively
   at larger iterations.
   }
    \label{fig:bayesgmm3}
\end{figure}

\newpage

\bibliographystyle{abbrv} 
\bibliography{sampling}

\end{document}